\documentclass[11pt,a4paper]{article}
\usepackage[top=1in,bottom=1in,left=0.9in,right=0.9in]{geometry}
\usepackage{amsmath,amssymb,amsfonts,amsthm,mathtools}
\usepackage[utf8]{inputenc}
\DeclareUnicodeCharacter{0166}{\text{\TH}}
\usepackage[T1]{fontenc}
\usepackage{soul}
\usepackage{enumitem}
\usepackage{fix-cm}
\usepackage{booktabs}
\usepackage{nicefrac}
\usepackage[expansion=false]{microtype}
\usepackage{xcolor}
\usepackage{bbm}
\usepackage{placeins}
\usepackage{graphicx}
\usepackage{subcaption}
\usepackage{multirow}
\usepackage{algorithm}
\usepackage{algorithmic}
\usepackage{tikz}
\usetikzlibrary{arrows.meta,positioning,fit,calc}
\usepackage[round,authoryear]{natbib}
\usepackage{hyperref}
\usepackage{url}
\usepackage[capitalize,noabbrev]{cleveref}
\crefname{equation}{Equation}{Equations}
\Crefname{equation}{Equation}{Equations}
\crefformat{equation}{Equation~#2(#1)#3}
\Crefformat{equation}{Equation~#2(#1)#3}
\hypersetup{colorlinks=true,citecolor=blue,linkcolor=blue,urlcolor=blue}

\usepackage{amsmath,amsfonts,bm}

\def\eqref#1{equation~\ref{#1}}

\def\1{\bm{1}}

\def\va{{\bm{a}}}

\def\ve{{\bm{e}}}

\def\vh{{\bm{h}}}

\def\vs{{\bm{s}}}

\def\vu{{\bm{u}}}

\def\vz{{\bm{z}}}

\DeclareMathAlphabet{\mathsfit}{\encodingdefault}{\sfdefault}{m}{sl}
\SetMathAlphabet{\mathsfit}{bold}{\encodingdefault}{\sfdefault}{bx}{n}

\newcommand{\softmax}{\mathrm{softmax}}

\renewcommand{\eqref}[1]{\cref{#1}}
\newcommand{\TabCon}{\textsc{TabCon}}
\newcommand{\TabPFN}{\textsc{TabPFN}}
\newcommand{\TabICL}{\textsc{TabICL}}
\newcommand{\SOMTab}{\textsc{SOMTab}}
\newcommand{\LimiX}{\textsc{LimiX}}
\newtheorem{prop}{Proposition}

\title{Can Tabular Foundation Models Amortize Statistical Inference?}

\author{%
\parbox{0.95\textwidth}{\centering
Kai Ye$^{1}$, Shijin Gong$^{2}$, Hongyi Zhou$^{3}$,\\
Valentina Zangirolami$^{4,*}$, Chengchun Shi$^{1,*}$\\[0.6em]
{\small
$^{1}$Department of Statistics, London School of Economics and Political Science, London, UK\\
$^{2}$School of Management, University of Science and Technology of China, Hefei, China\\
$^{3}$School of Statistics and Data Science, Shanghai University of Finance and Economics, Shanghai, China\\
$^{4}$Department of Economics, Management and Statistics, University of Milano-Bicocca, Milan, Italy\\[0.3em]
$^{*}$Corresponding authors. Emails:
\texttt{valentina.zangirolami@unimib.it}, \texttt{c.shi7@lse.ac.uk}
}
}}
\date{}

\begin{document}
\maketitle

\begin{abstract}
\noindent
For decades, statistical inference has largely been developed one problem at a time. Given a scientific target, such as a treatment effect or a regression function, statisticians design a problem-specific estimator together with a procedure for quantifying its uncertainty. This paper proposes a different paradigm. We focus on a classical problem in statistical inference, confidence interval  construction, and develop \TabCon, an amortized inference system built on a tabular foundation model that produces confidence intervals for new datasets through a simple forward pass. The key methodological ingredients of TabCon are a sparse mixture-of-experts architecture and reinforcement-learning-based post-training that calibrate the resulting confidence intervals to a desired coverage level. Across a wide range of benchmark datasets, TabCon attains near-nominal coverage while producing short confidence intervals. At inference time, it also offers considerably greater computational efficiency, running {\bf{50}} times faster than the classical bootstrap procedure, even when the latter uses only 50 bootstrap samples.
\end{abstract}

\section{Introduction}\label{sec:intro}
Modern machine learning has largely focused on achieving accurate prediction, yet this alone is often insufficient for scientific analysis. To draw scientific conclusions, researchers must also quantify how precisely the data identify a quantity of interest, whether an average treatment effect, a regression coefficient, or the value of a regression function. Quantifying such uncertainty is a long-standing and fundamental problem in statistics. Over the past century and beyond, statisticians have developed a range of inferential procedures for uncertainty quantification, together with rigorous theoretical guarantees that characterize when and why these procedures are valid. This line of work has given rise to a rich and mature literature in statistical inference\footnote{The term inference carries different meanings in statistics and machine learning. In statistics, it refers to uncertainty quantifying through procedures such as confidence interval construction and hypothesis testing. In machine learning, inference usually refers to applying a trained model to new inputs at deployment time.} (see Section \ref{sec:related} for a more detailed discussion).

These inferential procedures are often non-amortized: given a new dataset, one must run a dedicated procedure at inference time for uncertainty quantification. This paradigm, however, comes with several limitations. Some methods rely on restrictive parametric assumptions or large-sample approximations for their theoretical validity, while others incur considerable computational cost because they require repeated resampling and model refitting.

We propose a different paradigm: amortizing statistical inference through pre/post-training. As an initial step toward this goal, this paper studies a classical problem in statistical inference, constructing confidence intervals (CIs) for the value of an unknown regression function at a given covariate point. 

\textbf{Our main contribution} is to introduce \TabCon, a tabular foundation model that amortizes confidence interval construction across datasets, producing CIs for each new dataset through a single forward pass. This design substantially reduces computation at inference time, making \TabCon\ {\bf 50} times faster than the classical bootstrap procedure at inference time in our experiments (Table \ref{tab:runtime-grid-new-task}). Compared with existing tabular foundation models,  \TabCon\ contains three novelties:
\begin{enumerate}[leftmargin=*]
    \item \TabCon\ implements a novel methodology for CI construction by employing residual heads that directly model the distribution of the estimation error, rather than the predictive distribution of the outcome as in existing tabular foundation models. This difference enables \TabCon\ to construct CIs for the regression function, rather than prediction intervals for future outcomes. In our experiments, the latter are {\bf{13}}\%–{\bf{148}}\% wider than the CIs produced by \TabCon. 
    
    \item From an architectural perspective, \TabCon\ adopts a sparse mixture-of-experts (MoE) structure \citep{shazeer2017outrageously}, with different experts specializing in data generating processes with different levels of difficulty for CI construction (see \Cref{sec:setup}). Our ablation study demonstrates the crucial role of this architecture in calibrating the resulting CIs to a desired coverage level (Figure \ref{fig:component-ablation}).

    \item From a training perspective, \TabCon\ differs from existing tabular foundation models, which are predominantly trained with supervised prediction objectives \citep{hollmann2025tabpfn,qu2026tabiclv2}, and instead employs reinforcement learning (see \Cref{sec:method}). This allows the adaptive selection of a small subset of experts that are best suited to constructing CIs for the current dataset.
\end{enumerate}

\section{Related Works}\label{sec:related}
Our work is closely related to two lines of research: statistical methods for confidence interval construction, and tabular foundation models.


\textbf{Statistical methods for CI construction}. CI construction is a classical statistical inference problem. Conventional approaches often impose a simple parametric structure on the underlying relationship, such as a linear regression model, so that uncertainty in the estimated regression function can be characterized through uncertainty in the estimated model parameters. Under classical large-sample asymptotic theory, including $M$-estimation \citep{van1996weak,van2000asymptotic}, these parameters are asymptotically normal under mild regularity conditions with a closed-form asymptotic variance. Combining this normal approximation with a plug-in estimate of the corresponding variance leads to Wald-type CIs of the estimated parameter \citep{wald1943tests}, which can be translated into CIs for the regression function at a given covariate through the delta method \citep{casella2024statistical}. 

The second type of approaches is resampling-based. These approaches construct many pseudo datasets through resampling, repeatedly apply the estimation procedure to each one, and use the empirical distribution of the resulting estimates for CI construction. The main example is the bootstrap \citep{efron1979bootstrap}, with numerous variants \citep[e.g.,][]{wu1986jackknife,efron2000introduction}.


These two approaches represent the main paradigms for CI construction, yet each suffers from certain limitations. Conventional asymptotic methods rely on restrictive parametric assumptions. Resampling-based methods are generally more flexible, but they can fail for complex nonparametric estimators, such as regression trees \citep{Mouli2007Conf}. 


Other approaches provide uncertainty quantification for complex nonparametric models. Conformal prediction offers prediction intervals with finite-sample coverage \citep{bates2021distribution,gibbs2025conformal}, with an inferential target different from ours: it quantifies uncertainty in a future outcome rather than in an estimated regression function. The former combines the outcome variability with estimation uncertainty arising from finite data. Consequently, prediction intervals are generally substantially wider than CIs. Other methods have also been designed for CIs tailored to specific nonparametric estimators, including random forests \citep{JMLR:v17:14-168} and neural networks \citep{schupbach2020quantifying,fei2024u,meng2026inference}. While sharing our inferential target, they are estimator-specific and non-amortized, requiring a dedicated procedure for each new dataset. 

To the contrary, \TabCon\ produces CIs directly through a single forward pass of a tabular foundation model, a class of models that we review next. 

\textbf{Tabular foundation models}. \TabCon\ is motivated by the recent development of tabular foundation models, which amortize the prediction procedure across datasets. These models are pretrained on diverse datasets generated from a wide range of data generating processes and make predictions on a new dataset through in-context learning. \TabPFN\ established this approach for tabular classification \citep{hollmann2023tabpfn} and subsequently extended it to regression \citep{hollmann2025tabpfn}. Later developments, such as \TabICL, \LimiX\ and \SOMTab, introduce more efficient model architectures and training strategies to scale the model to larger datasets  \citep{feuer2024tunetables,qu2025tabicl,zhang2025limix,qu2026tabiclv2,wang2026somtab}. This paradigm has also been extended beyond standard supervised prediction to a broader range of tasks, including causal inference \citep{ma2026foundation,zhang2026}, clustering \citep{zhao2026tabclustpfn}, and contextual bandits \citep{tan2026pfn}.

Importantly, standard tabular foundation models such as \TabPFN\ and \TabICL\ already provide a form of uncertainty quantification by learning a predictive distribution for a future outcome, rather than producing only a point prediction. However, as discussed earlier, uncertainty about a future outcome is fundamentally different from our target, the uncertainty about the estimated regression function. Prediction intervals derived from these models can therefore be substantially wider than CIs for the regression function itself (see Section \ref{sec:experiment}).

More recently, a few papers have studied uncertainty in estimating unknown quantities. \citet{nagler2025martingale} combine \TabPFN\ predictions with martingale posterior sampling to quantify uncertainty in conditional means and quantiles. \citet{ng2025tabmgp} similarly uses \TabPFN's predictive resampling distribution to construct posteriors for scientific estimands. \citet{fortini2026uncertainty} infer the epistemic uncertainty of a tabular foundation model's predictions from its sequential predictive updates. These methods target uncertainty closely related to our inferential objective, but recover such uncertainty through additional inference-time sampling or sequential predictive computation. In contrast, \TabCon\ amortizes CI construction directly  through pretraining and reinforcement learning (RL)-based post-training, without inference-time sampling or sequential predictive computation.

\section{Problem Setup and TabCon Overview}\label{sec:setup}
In this section, we first formulate the confidence interval construction problem. We then discuss main idea behind \TabCon\ and introduce its architecture. 

\textbf{Problem setup}. Consider a dataset $\mathcal{D}=\{(X_i,Y_i)\}_{i=1}^n$ consisting of i.i.d. predictor-response pairs, where each $X_i$ is a vector-valued covariate and $Y_i$ is a scalar response. Let
$f(x)=\mathbb{E}(Y| X=x)$ 
denote the conditional mean function. Our goal is to infer the value of $f$ at a collection of future covariate values $\{X_{n+i}\}_{i=1}^m$. In many scientific applications, point estimation alone is insufficient. For each future covariate value $x$ and a prescribed significance level $\alpha\in(0,1)$, we seek a confidence interval $[L,U]$ satisfying
\begin{equation}\label{CI:coverage}
\mathbb{P}\bigl(L\le f(x)\le U\bigr)\ge 1-\alpha,
\end{equation}
where the probability is taken with respect to the randomness of the interval endpoints $L$ and $U$ constructed from the finite-sample dataset $\mathcal{D}$.

Our primary objective is to satisfy, or closely approximate, the coverage requirement in \eqref{CI:coverage}. Among CIs achieving the desired coverage level, we further seek intervals with small expected length $\mathbb{E}(U-L)$, in order to localize the unknown value $f(x)$ as precisely as possible.

\textbf{Amortized CI construction}. We next describe how \TabCon\ achieves amortized CI construction. As discussed in \Cref{sec:related}, several recent methods use tabular foundation models for CI construction. However, they use the model primarily as a predictive engine and recover CIs through additional inference-time sampling or computation. As a result, their CI construction procedure is not fully amortized through training.

\TabCon's main idea is simple. We first use an existing tabular foundation model to obtain a point estimate $\widehat{f}(x)$ of the regression function. During training, the data generating process (DGP) is known, and hence the true value $f(x)$ is available. This allows us to compute the estimation error $\widehat{f}(x)-f(x)$. We then adopt the in-context learning paradigm of existing tabular foundation models, but change the learning target from the future response to this estimation error. In this way, \TabCon\ learns the distribution of the estimation error across pretraining tasks.

Let $Q_{\alpha}(x)$ denote the $\alpha$th quantile of the distribution of $\widehat{f}(x)-f(x)$, a $(1-\alpha)$ confidence interval for $f(x)$ is given by
\begin{equation}\label{eqn:proposedCI}
\left[\widehat{f}(x)-Q_{1-\alpha/2}(x), \widehat{f}(x)-Q_{\alpha/2}(x)\right].
\end{equation}
In practice, \TabCon\ replaces the unknown quantiles with their learned estimates. The following proposition establishes the coverage validity of this oracle construction.

\begin{prop}
The CI in \eqref{eqn:proposedCI} achieves coverage probability $1-\alpha$ for $f(x)$.
\end{prop}

The proof follows directly from the definition of the quantiles and is provided in Appendix \ref{app:proof}. After pretraining, we further employ post-training to calibrate the learned quantiles toward the desired coverage level while reducing interval length; see \Cref{sec:method} for details.

\begin{figure*}[t]
\centering
\includegraphics[width=\linewidth]{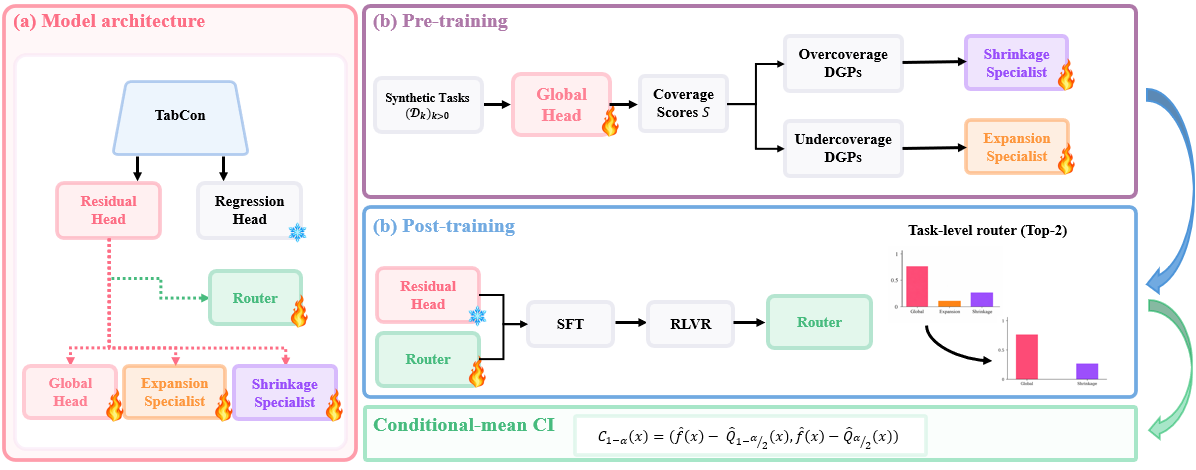}
\caption{An overview of \textbf{TabCon}. Panel (a) illustrates the model architecture, which consists of (i) a frozen \ul{\textit{regression head}} that produces point estimates of the regression function, (ii) three trainable residual heads that model the estimation error distribution, including a \ul{\textit{global head}} and two \ul{\textit{specialists}}, and (iii) a trainable \ul{\textit{router}} that determines how the residual heads are combined. Panels (b) illustrate the training and inference pipeline, which follows the pretraining and post-training paradigm of modern foundation models. At inference time (bottom), the router weights combine the outputs of the residual heads to 
construct the CI according to \eqref{eqn:proposedCI}.}
\label{fig:tabcon-ci}
\end{figure*}

\textbf{Model architecture}. \TabCon\ consists of three components: (i) a regression head that produces a point estimate of the regression function; (ii) three residual heads that estimate distribution of the estimation error, with each head specializing in a different class of data generating processes; and (iii) a router that adaptively selects the residual heads most appropriate for the current dataset. Figure \ref{fig:tabcon-ci}(a) provides a graphical overview of the model architecture. We describe each component in detail:
\begin{enumerate}[leftmargin=*]
    \item[(i)] \ul{\textit{The regression head}} takes the observed dataset $\mathcal{D}$ and a set of future covariate values as input, and produces a point estimate $\widehat{f}(x)$ for each queried covariate value $x$. It can either be a frozen pretrained tabular foundation model or trained jointly with the rest of \TabCon. In our implementation, we use a frozen TabPFN-v3 model as the regression head.
    
    \item[(ii)] \ul{\textit{The residual heads}} take the dataset $\mathcal{D}$ and a set of future covariate values as input, and output, for each queried covariate value, a discretized distribution over the estimation error represented by a set of bins. \TabCon\ contains three residual heads: (a) a global head, (b) an expansion specialist, and (c) a shrinkage specialist. All three share the TabPFN-v3 architecture and are trained to model the estimation error, but on different classes of DGPs: (a) The global head is trained on the full collection of DGPs. (b) The expansion specialist is trained primarily on DGPs for which the baseline CI produced by the global head tends to undercover, with the goal of widening the interval toward the nominal coverage level. (c) Conversely, the shrinkage specialist is trained primarily on DGPs for which the baseline CI tends to overcover, with the goal of shortening the interval while preserving the desired coverage level. 
    
    \item[(iii)] \ul{\textit{The router}} is another transformer model that takes as input the final-layer representations from all residual heads across the queried covariate values and outputs a (sparse) weight vector used to combine their predictions. It employs attention to aggregate information across queried covariates into a set of routing logits, which are then transformed through a softmax function to obtain the mixing weights. To save space, we defer the detailed network architecture to Appendix \ref{app:router_details}. 
\end{enumerate}
We describe the training procedures for the residual heads and the router in the next section. 

\section{Model Training}\label{sec:method}
\textbf{Overview}. The training of \TabCon\ follows a two-stage pretraining and post-training paradigm, similar in spirit to the training pipeline of modern large language models (LLMs). During \textit{pretraining}, we train the residual heads to model the distribution of estimation errors using a supervised objective. During \textit{post-training}, we train the router while keeping the residual heads fixed. Post-training itself consists of two phases: we first initialize the router through \textit{supervised fine-tuning} (SFT), using a loss function similar to the pretraining objective, and then further optimize the routing policy through \textit{reinforcement learning with verifiable rewards} (RLVR) to directly target the performance of the resulting CIs, improving coverage while potentially reducing interval length. We detail each training stage below.

\textbf{Pretraining}. Following TabPFN, we construct a diverse collection of simulated DGPs and generate 50,000 datasets for pretraining. The DGPs are designed following \citet{grinsztajn2025tabpfn}. We first train the global head on the 50,000 datasets to learn the distribution of the estimation error. To represent this distribution, we discretize the estimation error over a wide range of bins and train the model to predict the probability mass assigned to each bin. The parameters of the global head are optimized by minimizing
\begin{equation}\label{eqn:negative}
\ell_{\mathrm{global}}
=\ell_{\mathrm{NLL}}+\eta_1 \ell_{\mathrm{CRPS}},
\end{equation}
where $\ell_{\mathrm{NLL}}$ denotes the negative log-likelihood function (i.e., the cross-entropy loss) induced by the discretized estimation error distribution.   The second term $\ell_{\mathrm{CRPS}}$ in \eqref{eqn:negative} denotes the continuous ranked probability score \citep[CRPS,][]{gneiting2007strictly}, which measures the discrepancy between the predicted cumulative distribution and the distribution induced by the realized estimation error. The tuning parameter $\eta_1\geq 0$ controls the contribution of the CRPS term to the overall loss. 

Optimizing \eqref{eqn:negative} yields a learned discrete distribution over the bins, denoted by $p_{\textrm{global}}$, from which we recover the corresponding quantile function $Q_{\alpha}$. We next use these learned quantiles to categorize all DGPs into three coverage regimes, in order to train the two specialist heads. Specifically, for each DGP and significance level $\alpha$, we plug the estimated quantiles into \eqref{eqn:proposedCI} to construct CIs and evaluate their empirical coverage, averaged across datasets and future covariate values. Denote the resulting average coverage by $C_{\alpha}$. 
We repeat this procedure over a sequence of significance levels $\{\alpha_j\}_j$ and compare the average coverage $C_{\alpha_j}$ at each level with its nominal target $1-\alpha_j$. We then aggregate these deviations into a coverage score
\begin{equation*}
S=\sum_j \lambda_j\left[C_{\alpha_j}-(1-\alpha_j)\right],
\end{equation*}
where $\lambda_j\ge 0$ denotes the weight and satisfies $\sum_j\lambda_j=1$. In particular, we assign larger weights to smaller significance levels to prioritize coverage accuracy at higher nominal coverage levels.

The score $S$ summarizes the overall coverage performance of a DGP: lower values indicate undercoverage, whereas higher values indicate overcoverage. We rank all DGPs according to $S$ and partition them into three groups. DGPs below the $1/3$ quantile of the score distribution define the undercoverage regime and are primarily used to train the expansion specialist, whereas those above the $2/3$ quantile define the overcoverage regime and are primarily used to train the shrinkage specialist. 

Specifically, we consider the following loss function for training each specialist,
\begin{eqnarray}\label{eqn:specialist}
\ell_{\mathrm{specialist}}=\ell_{\mathrm{NLL}}
+\eta_1 \ell_{\mathrm{CRPS}}+\eta_2 \textrm{KL}(p_{\textrm{global}},p_{\textrm{specialist}}).
\end{eqnarray}
The objective function in \eqref{eqn:specialist} is closely related to that in \eqref{eqn:negative}, but differs in two important ways: 
\begin{enumerate}[leftmargin=*]
    \item[(i)] The negative log-likelihood and CRPS terms are evaluated only on datasets from the DGP regime associated with the corresponding specialist. For example, the expansion specialist is trained on datasets generated from DGPs whose coverage scores fall below the $1/3$ quantile of the score distribution. This specialization allows each head to learn the estimation error distribution more accurately within its corresponding coverage regime.
    \item[(ii)] Equation \ref{eqn:specialist}  includes an additional Kullback-Leibler (KL) regularization term, which is evaluated on datasets generated from the full collection of DGPs. This term penalizes excessive deviation of the specialist distribution $p_{\textrm{specialist}}$ from the global distribution $p_{\textrm{global}}$, in order to preserve the generalization ability inherited from the global head. The tuning parameter $\eta_2>0$ controls the strength of this regularization. Similar KL regularization is widely used in post-training LLMs \citep[e.g.,][]{ouyang2022training}.
\end{enumerate}

\textbf{Post-training}. During post-training, we keep all residual heads fixed and optimize only the parameters of the router, whose output is a three-dimensional weight vector $(w_1,w_2,w_3)$ satisfying $w_j\geq 0$ and $\sum_{j=1}^3 w_j=1$. These weights combine the three residual distributions to form the final estimation error distribution
\begin{eqnarray}\label{eqn:finalprob}
p=w_1 p_{\textrm{global}}+w_2 p_{\textrm{expansion}}+w_3 p_{\textrm{shrinkage}},
\end{eqnarray}
where, with a slight abuse of notation, $p_{\textrm{expansion}}$ and $p_{\textrm{shrinkage}}$ denote the learned probability mass functions of the estimation error produced by the two specialist heads. 

Post-training consists of two phases. In \ul{\textit{Phase I}}, we again generate 50,000 datasets from the full collection of DGPs and train the router through SFT by minimizing
\begin{eqnarray}\label{eqn:SFTloss}
\ell_{\textrm{SFT}}
=
\ell_{\textrm{NLL}}
+
\eta_3 \ell_{\textrm{CRPS}},
\end{eqnarray}
where $\eta_3>0$ is a tuning parameter. This objective is again, very similar to the pretraining objective in \eqref{eqn:negative}, except that both the negative log-likelihood and CRPS are now evaluated with respect to the final estimation error distribution $p$ in \eqref{eqn:finalprob}. Because $p$ is determined by the router weights, minimizing this objective directly trains the router to properly combine the three residual heads.

However, the SFT objective suffers from two important limitations. (i) Minimizing \eqref{eqn:SFTloss} does not necessarily drive any routing weights to zero and therefore does not enforce sparse routing, which we find crucial for achieving the desired coverage. (ii) Although the negative log-likelihood and CRPS measure the quality of the learned estimation error distribution, they do not directly optimize the coverage probability and interval length, by which CIs are ultimately evaluated. These limitations motivate a second post-training phase based on RL, which allows us to optimize a sparse routing policy toward these inferential objectives.

Specifically, in \ul{\textit{Phase II}}, we generate 100,000 datasets and optimize the router parameters using an RLVR algorithm, RLOO \citep{ahmadian2024back}. In this RL formulation, the state corresponds to the current dataset, the action is the sparse routing weight vector that determines which residual head or heads are selected and how their outputs are combined, and the reward evaluates the quality of the resulting CI in terms of coverage and interval length. The router parameters define the policy over routing actions and are optimized to maximize the expected reward.

To address limitation (i), we explicitly sparsify each routing weight vector produced by the router. Specifically, we set its smallest component to zero and renormalize the remaining two components so that they sum to one. The resulting routing vector therefore always contains at least one zero entry.

To address limitation (ii), we design a reward function with the following four components:  
\begin{enumerate}[leftmargin=*]
    \item We reward reductions in interval length produced by the global head, encouraging shorter confidence intervals whenever possible.
    \item We penalize deviations of the empirical coverage from its nominal level using an asymmetric criterion that assigns a substantially larger penalty to undercoverage than to overcoverage.
    \item We impose an additional penalty whenever the resulting coverage is worse than that achieved by the global head, preventing the routing policy from worsening the coverage performance.
    \item Finally, we penalize highly asymmetric CIs, in which the lower and upper endpoints correspond to strongly unbalanced tail probabilities. 
\end{enumerate}
Due to space constraints, we defer the detailed specification of the reward function to Appendix \ref{app:pt_details}.

\section{Experiments}\label{sec:experiment}
\textbf{Datasets}. We conduct extensive experiments in this section to evaluate the performance of \TabCon\ on both synthetic and semi-synthetic data. For the synthetic evaluation, we consider a rich collection of 64 DGPs. Among them, 32 are also employed during the training of \TabCon\ and are used for in-distribution evaluation, while the remaining 32 are unseen during training and are used for out-of-distribution evaluation. Importantly, even for the in-distribution setting, all evaluation datasets are generated independently from the training data, so {\bf no} test dataset is observed during training.

We further divide the synthetic DGPs into two benchmark settings: a \ul{\emph{standard}} setting covering a broad range of commonly encountered DGPs, and a \ul{\emph{stress-test}} setting consisting of deliberately challenging DGPs designed to evaluate the robustness of \TabCon\ under more difficult data regimes. 

Finally, we construct a semi-synthetic benchmark based on real feature matrices from 13 OpenML-\ul{\textit{CC18}} datasets \citep{bischl2021openml}. We retain these covariates but generate responses from synthetic regression functions with noise. This semi-synthetic design combines realistic covariates from real-world datasets with known regression functions, the latter of which allow us to evaluate CI coverage against the ground truth.

\textbf{Baseline algorithms}. We compare \TabCon\ against five methods for CI construction, including two based on normal approximations and three built on TabPFN that represent different approaches to uncertainty quantification with existing tabular foundation models. Specifically:  
\begin{enumerate}[leftmargin=*]
    \item \ul{\textit{Linear}} fits a linear regression model, estimates the covariance matrix of the regression coefficients, and constructs Wald-type CIs for the regression function. 
    \item \ul{\textit{RFF}} augments the covariates with random Fourier features to accommodate nonlinear relationships and then applies the same Wald-type CI construction as Linear.
    \item \ul{\textit{Bootstrap}} combines the standard nonparametric bootstrap \citep{efron1979bootstrap} with TabPFN-v3 point estimation. To reduce inference-time computation, we generate only 50 bootstrap samples for each observed dataset, yielding 50 TabPFN-v3 point estimates. We then use their empirical distribution to construct the corresponding CI.
    \item \ul{\textit{PI}} directly uses the predictive distribution returned by TabPFN-v3. We extract its $\frac{\alpha}{2}$ and $(1-\frac{\alpha}{2})$ quantiles to construct a prediction interval, which is then used as a CI for the regression function.
    \item \ul{\textit{AMP}} adopts the approximate martingale posterior approach of \citet{nagler2025martingale} that combines TabPFN-v3's prediction with martingale posterior distribution (MPD) to construct CIs. Following \citet{nagler2025martingale}, we use 50 chains with 50 forward samples per chain.
\end{enumerate}
In our ablation study, we further compare \TabCon\ with two variants: one constructs CIs using only the pretrained global head (denoted by \ul{\textit{Global}}), while the other uses the mixture distribution learned after the SFT stage (denoted by \ul{\textit{SFT}}).

\textbf{Evaluation metrics}. We evaluate each CI construction method in terms of coverage probability (CP), interval length (IL), and inference time. In this section, we focus on the setting with significance level $\alpha=5\%$, corresponding to a nominal coverage level of $95\%$. Results for $\alpha=1\%$ and $\alpha=10\%$, corresponding to nominal coverage levels of $99\%$ and $90\%$, respectively, are reported in Appendix \ref{app:add_emp_results}.

\textbf{Coverage and interval length}. In Table \ref{tab:tabcon-main-results}, we report the average CP and IL across queried covariate values and evaluation datasets under four settings: standard in-distribution, standard out-of-distribution, stress-test, and CC18. We additionally visualize CP and IL separately for each DGP in the stress-test setting in Figure \ref{fig:stress-radar-95}, with detailed descriptions of the DGPs and the corresponding CP and IL values provided in Appendix \ref{app:impl_details}.

\begin{table}[t]
    \centering
    \caption{Coverage probability (CP) and interval length (IL) at the 95\% nominal coverage level. From left to right, 
    results are reported for the standard in-distribution (ID), standard out-of-distribution (OOD),
    stress-test, and CC18 settings. For CP, the value closest to the nominal level of 0.95 is shown in \textbf{bold}, and the second closest is \ul{underlined}.
    For IL, among methods satisfying $\mathrm{CP}\geq 0.95$, the shortest interval is shown in \textbf{bold}, and the second shortest is \ul{underlined}.}
    \label{tab:tabcon-main-results}

    \setlength{\tabcolsep}{6pt}
    \renewcommand{\arraystretch}{1.10}
    \small

    \begin{tabular}{lcccccccc}
        \toprule
        Method 
        & \multicolumn{2}{c}{Standard (ID)}
        & \multicolumn{2}{c}{Standard (OOD)}
        & \multicolumn{2}{c}{Stress-test}
        & \multicolumn{2}{c}{CC18} \\
        \cmidrule(lr){2-3}
        \cmidrule(lr){4-5}
        \cmidrule(lr){6-7}
        \cmidrule(lr){8-9}
        & CP & IL
        & CP & IL
        & CP & IL
        & CP & IL \\
        \midrule

        Linear
        & 0.8653 & 8.5699
        & 0.6814 & 2.5411
        & 0.7673 & 6.9267
        & 0.5377 & 78.0342 \\

        RFF
        & 0.8511 & 4.0549
        & 0.8029 & 2.5039
        & 0.7381 & 2.6404
        & 0.6238 & 0.7514 \\

 PI
        & 0.9792 & \ul{4.0261}
        & 0.9695 & \ul{2.1793}
        & 0.9902 & \ul{3.2091}
        & 0.9953 & 2.6531 \\

        Bootstrap
        & 0.5935 & 0.9200
        & 0.6820 & 0.8107
        & 0.6970 & 0.6651
        & 0.7328 & 0.4084 \\

        AMP
        & \ul{0.9369} & 3.9612
        & \ul{0.9351} & 2.0706
        & \ul{0.9729} & 3.2663
        & \ul{0.9818} & \ul{2.4117} \\

        TabCon
        & \textbf{0.9548} & \textbf{2.6824}
        & \textbf{0.9531} & \textbf{1.9288}
        & \textbf{0.9561} & \textbf{1.8581}
        & \textbf{0.9654} & \textbf{1.0686} \\

        \bottomrule
    \end{tabular}
\end{table}

\begin{figure}[t]
    \centering
    \includegraphics[width=0.85\linewidth]
    {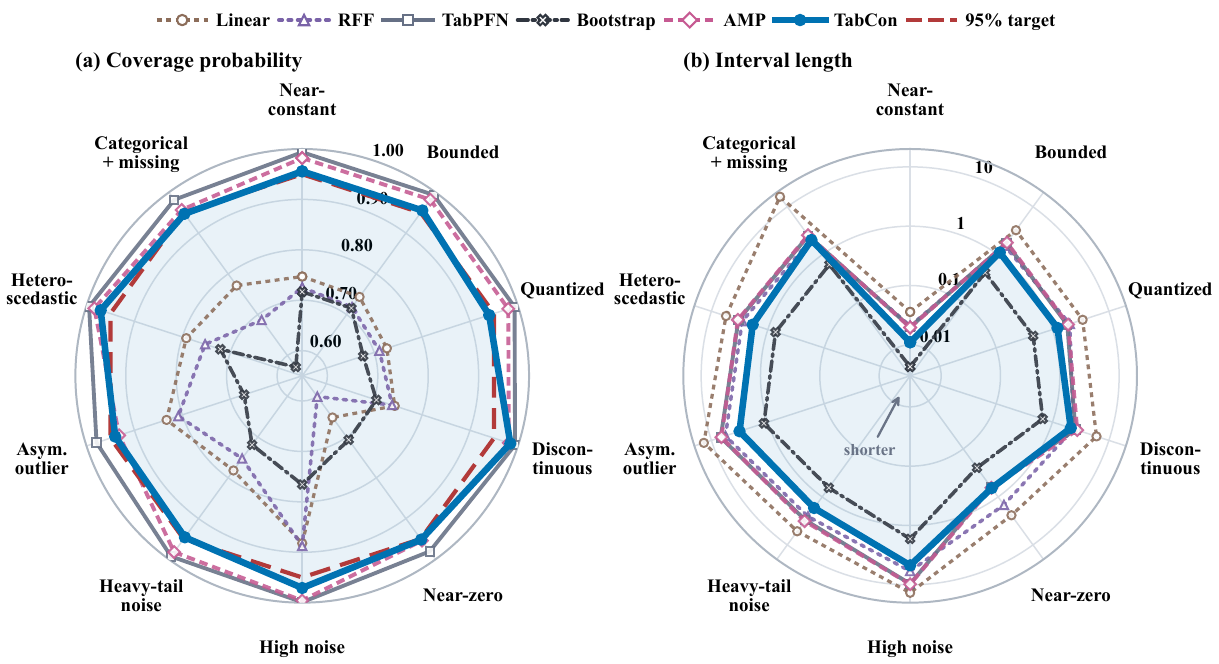}
    \caption{Radar plots of CP (left) and IL (right) across ten stress-test DGPs at the 95\% nominal coverage level. In the left panel, the radial levels represent coverage from 60\% to 100\%, increasing from the center outward. The dashed red polygon marks the nominal 95\% coverage level, with curves closer to this reference indicate better coverage. In the right panel, the radial levels correspond to IL values from 0.01 to 10 on a logarithmic scale. Curves closer to the center indicate shorter intervals.}
    \label{fig:stress-radar-95}
\end{figure}

We make the following observations:
\begin{enumerate}[leftmargin=*]
    \item \TabCon\ delivers the most competitive performance across all evaluation settings. As shown in Table \ref{tab:tabcon-main-results}, it achieves coverage closest to the nominal level of 95\% in every setting while maintaining empirical coverage above 95\%. Looking across individual DGPs in the stress-test setting, \TabCon\ slightly undercovers in a few cases, but its coverage remains above 94\% throughout (Figure \ref{fig:stress-radar-95}). Meanwhile, among methods achieving sufficient coverage, \TabCon\ produces the shortest intervals in nearly all settings, striking a favorable balance between coverage and interval length. 
    \item Normal approximation methods (Linear and RFF) perform poorly. They substantially undercover in both standard settings, CC18 and across nearly all individual DGPs in the stress-test setting. One important source of this undercoverage is approximation bias induced by the linear working models: misspecification of the regression function can lead to biased point estimates that are not accounted for by the resulting Wald intervals \citep{white1982maximum}. This interpretation is also consistent with the high-noise setting (Figure \ref{fig:stress-radar-95}), where coverage increases to approximately $0.88$ for both methods. As the noise level increases, the resulting intervals become wider and the approximation bias becomes smaller compared with the overall sampling uncertainty. 
    \item Bootstrap methods also substantially undercover in most settings. This suggests that classical resampling procedures do not automatically inherit their validity when combined with modern tabular foundation models. One plausible explanation is that the highly nonlinear nature of the TabPFN estimator makes its sampling distribution difficult to approximate through naive resampling, resulting in intervals that are systematically too narrow.
    \item PI, by contrast, overcovers in every setting and produces intervals that are {\bf{13}}\%–{\bf{148}}\% longer than those produced by \TabCon. This is expected because TabPFN’s predictive distribution captures not only uncertainty in the estimated regression function but also outcome variability.
    \item Among the five baselines, AMP performs the strongest overall, but it tends to undercover in the standard settings and overcover in CC18 as well as many individual DGPs in the stress-test setting. Moreover, its intervals are {\bf 7}\%–{\bf{126}}\% longer than those produced by \TabCon.
\end{enumerate}

\textbf{Computational efficiency}. We report the median inference time of \TabCon\ and the baseline methods in Table~\ref{tab:runtime-grid-new-task}. All runtimes are measured on an NVIDIA RTX 6000 Ada Generation GPU for context sizes
$C\in\{64,128,256,512,1{,}024\}$ and query-batch sizes
$Q\in\{64,256,1{,}024\}$. As expected, the simple normal approximation baselines are computationally efficient. However, as shown above, this benefit comes with substantial undercoverage and, in some settings, excessively long intervals, limiting their practical utility.

Among methods based on tabular foundation models, \TabCon\ has an inference time comparable to that of PI based on TabPFN-v3. Moreover, the inference times of both methods remain nearly constant as the numbers of context and query points increase. As shown earlier, however, \TabCon\ provides substantially better coverage and much shorter intervals than PI. A more detailed decomposition of the runtime is provided in
Appendix~\ref{app:runtime-decomposition}. 

AMP is more computationally expensive, requiring approximately ${\bf 25}$\%--${\bf 108}\%$ more times than \TabCon\ across the three query sizes. Its runtime also increases substantially as the number of query points grows from $256$ to $1{,}024$, reflecting the additional predictive sampling required. 

Finally, the bootstrap is by far the most computationally expensive baseline. Despite using only 50 bootstrap refits, it is approximately ${\bf 50}$ times slower than \TabCon\ across the three query sizes.
\begin{table*}[t]
  \centering
  \caption{Median runtime in seconds, over 30 repetitions. $C$ and $Q$ denote the numbers of context and query points.}
  \label{tab:runtime-grid-new-task}
  \setlength{\tabcolsep}{4.0pt}
  \renewcommand{\arraystretch}{1.07}
  \scriptsize
  \begin{tabular}{rrrrrrrrrr}
    \toprule
    $C$ & $Q$ & Linear & RFF & PI & Global & Bootstrap & AMP & \TabCon \\
    \midrule
    64 & 64 & 0.0024 & 0.0025  & 0.2089 & 0.1978 & 10.5411 & 0.2579 & 0.2036 \\
     & 256 & 0.0024 & 0.0025 &  0.2108 & 0.1990 & 10.6450 & 0.2742 & 0.2072 \\
     & 1024 & 0.0026 & 0.0028 & 0.2185 & 0.2109 & 10.9722 & 0.4773 & 0.2297 \\
    \midrule
    128 & 64 & 0.0025 & 0.0028  & 0.2103 & 0.1986 & 10.5650 & 0.2595 & 0.2054 \\
     & 256 & 0.0025 & 0.0029  & 0.2119 & 0.1999 & 10.6264 & 0.2758 & 0.2082 \\
     & 1024 & 0.0026 & 0.0033 & 0.2214 & 0.2139 & 11.0762 & 0.4803 & 0.2320 \\
    \midrule
    256 & 64 & 0.0026 & 0.0063 & 0.2128 & 0.2010 & 10.6597 & 0.2619 & 0.2072 \\
     & 256 & 0.0026 & 0.0065 & 0.2137 & 0.2017 & 10.7106 & 0.2772 & 0.2097 \\
     & 1024 & 0.0027 & 0.0077 & 0.2236 & 0.2160 & 11.1982 & 0.4821 & 0.2340 \\
    \midrule
    512 & 64 & 0.0028 & 0.0108  & 0.2176 & 0.2060 & 10.9422 & 0.2670 & 0.2120 \\
     & 256 & 0.0028 & 0.0112  & 0.2187 & 0.2064 & 10.9773 & 0.2825 & 0.2146 \\
     & 1024 & 0.0029 & 0.0123 & 0.2265 & 0.2192 & 11.3492 & 0.4859 & 0.2376 \\
    \midrule
    1024 & 64 & 0.0031 & 0.0158 & 0.2211 & 0.2095 & 11.1274 & 0.2702 & 0.2155 \\
     & 256 & 0.0032 & 0.0159 & 0.2243 & 0.2129 & 11.3183 & 0.2882 & 0.2201 \\
     & 1024 & 0.0033 & 0.0175 & 0.2299 & 0.2224 & 11.5502 & 0.4890 & 0.2403 \\
    \bottomrule
  \end{tabular}
\end{table*}

\textbf{Ablation study}. Finally, we compare \TabCon\ with its two variants, Global and SFT. 
Figure \ref{fig:component-ablation} shows that both variants can substantially under- or overcover for several individual DGPs despite achieving near-nominal average coverage across all DGPs. Compared with the global head alone, SFT benefits from the specialist heads, but does not consistently calibrate the intervals. To the contrary, \TabCon\ adaptively increases coverage when the two variants undercover, while reducing coverage, and consequently IL, when they overcover. These results demonstrate the benefits of the specialist heads and RL-based post-training for calibrating the resulting CIs toward the nominal coverage level.

\begin{figure*}[t]
    \centering
    \includegraphics[width=\textwidth]{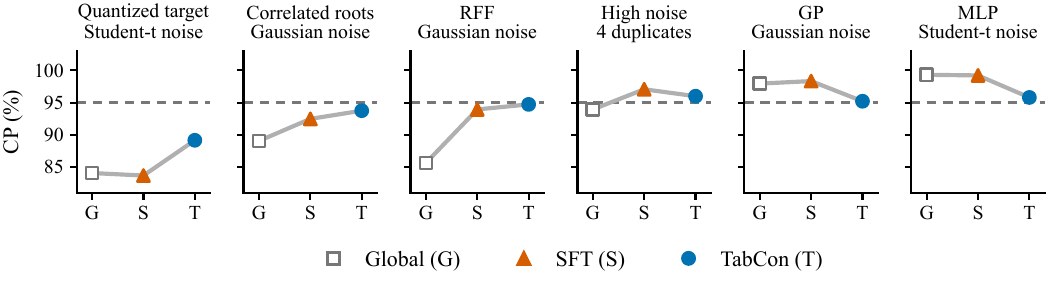}
    \caption{Ablation study on six selected DGPs in terms of coverage probability (CP). Dashed lines indicate the nominal coverage level of $95\%$.}
    \label{fig:component-ablation}
\end{figure*}

\section{Discussion}\label{sec:conclusion}
We presented \TabCon, a tabular foundation model for amortizing CI construction. Our results provide an initial demonstration that certain statistical inference tasks can be amortized through pretraining and post-training.  Statistical inference, however,  contains a much broader range of tasks beyond CI construction, including hypothesis testing. It remains an open question whether these tasks can similarly be amortized within a unified foundation model. 

\section*{AI use statement}

Generative AI tools were used to assist with code writing for technical implementation support and polish the writing of the manuscript. The development of the conceptual framework, the generation of the synthetic datasets, the design of the methodology and the experiments, data analysis and the interpretation of the results were developed and validated by the authors.  AI tools
were not used to formulate the mathematical results, to propose or refine research hypotheses, or to derive or write any of the proofs. The authors take full responsibility for all content of the paper.

\section*{Reproducibility statement}

To ensure the reproducibility of our results, we provide a detailed description of our experiments and methodological analysis in the main paper and the appendix. Section~\ref{sec:setup} introduces the construction of the confidence intervals and the overall model architecture used to estimate them directly. Section~\ref{sec:method} describes the training procedure, with particular emphasis on the pre-training and post-training stages. In Section \ref{sec:experiment}, we present the synthetic and semi-synthetic data-generation procedures together with the main experimental setup. Additional methodological details are provided in
Appendix \ref{app:alg_details}, including the router architecture and the design of the post-training procedure, while Appendix \ref{app:impl_details} reports the hyperparameters and implementation details of the methods considered in our experiments. Finally, the complete proof of our theoretical result is provided in Appendix~\ref{app:proof}.

\section*{Acknowledgements}
This work was supported by the UK Government’s AI Research Resource (AIRR) programme and undertaken using the Isambard-AI system.

\bibliography{references}
\bibliographystyle{plainnat}

\clearpage
\appendix
\setlength{\emergencystretch}{2em}

\section{Proof of Proposition 1}\label{app:proof}

\begin{proof}
Consider a covariate value $x$ and the estimation error $R(x) = \widehat{f}(x) - f(x)$. We begin by explicitly stating the decomposition of the true regression function $f(x)$, as
\begin{gather*}
    f(x) = \widehat{f}(x) - R(x),
\end{gather*}
where the first term on the RHS corresponds to the estimated regression function. Throughout, all probabilities are taken with respect to the randomness of $\mathcal{D}$. Let $F_x(r) = \mathbb{P}\bigl(R(x) \le r\bigr)$ denote the CDF of $R(x)$, and define the estimation error quantile function as
\begin{gather}
    Q_{\alpha}(x) = \inf\{r \in \mathbb{R} : F_x(r) \ge \alpha \}. \label{eq:quantile}
\end{gather}
Rearranging \eqref{CI:coverage} and the corresponding oracle confidence interval in \eqref{eqn:proposedCI}, we obtain
\begin{gather}
\mathbb{P}(\widehat{f}(x) - Q_{1- \alpha/2}(x) \le f(x) \le         \widehat{f}(x) - Q_{\alpha/2}(x))= \notag\\
    \mathbb{P}(Q_{\alpha/2}(x) \le R(x) \le Q_{1- \alpha/2}(x))= \notag\\
    F_x\bigl(Q_{1-\alpha/2}(x)\bigr) - F_x\bigl(Q_{\alpha/2}(x)\bigr). \label{eq:ci_cov}
\end{gather}

Assuming that the CDF $F_x$ is continuous at the relevant quantiles, the definition of $Q_\alpha(x)$ in \eqref{eq:quantile} implies
$F_x(Q_\alpha(x)) = \alpha$ for every $\alpha \in (0,1)$; see, e.g., \citet{van2000asymptotic}. Applying these two bounds to \eqref{eq:ci_cov} yields a coverage probability of $(1 - \alpha/2) - \alpha/2 = 1 - \alpha$.

\end{proof}

\section{Additional Methodological and Training Details}\label{app:alg_details}

This section provides additional details on the router architecture and its post-training procedure. After the Global and Specialized Heads are learned during pretraining, they are kept frozen throughout router optimization. The router relies only on the final query embeddings returned by the frozen \TabPFN-v3 backbone, which determine the task-dependent mixture weights assigned to the residual heads in \eqref{eqn:finalprob}.

A key feature of the architecture is that routing is performed at the task level, allowing the model to capture aspects of the complexity of the underlying regression problem. To this end, the representations across query points and ensemble members are progressively aggregated into a single task representation. We first describe this architectural construction and then provide further details on the post-training description in Section~\ref{sec:method}. The post-training procedure consists of SFT initialization in \ul{\textit{Phase I}} followed by reinforcement learning-based optimization in \ul{\textit{Phase II}}; here, we provide additional details on the reward used in the second phase to evaluate candidate routing actions across synthetic tasks.

\subsection{Router Architecture}\label{app:router_details}

The router architecture takes as input the embeddings produced by the frozen TabPFN-v3 ensemble members and outputs a sparse weight vector used to combine the residual distributions of the frozen residual heads (see paragraph \textbf{Model architecture} in Section \ref{sec:setup}). By aggregating information across both query and ensemble members dimensions, the router constructs a representation of the current task that allows the routing weights to adapt to task-specific complexity in CI construction. The router processes consist of five architectural blocks: (i) \textbf{Query-level encoding}, which maps each query representation into a shared router latent space through a learned encoder; (ii) \textbf{Query-set aggregation}, which summarizes the encoded representations obtained from (i) across the query dimension using multiple aggregation mechanisms; (iii) \textbf{Task-level aggregation}, which further aggregates information across the ensemble members to obtain a single task-level representation; (iv) \textbf{Routing probabilities}, which maps the task-level representation into a probability distribution over the three residual heads (i.e. the Global Head, Expansion Specialist, and Shrinkage Specialist); and (v) \textbf{Top-2 routing}, which sparsifies the resulting distribution by retaining only the two largest routing weights. We next describe each architectural block in detail.

\paragraph{Query-level encoding.} Starting from a collection of synthetic tasks generated from the DGPs described in Section~\ref{sec:method} and Appendix~\ref{app:impl_details}, with the number of generated datasets for each routing phase reported therein, we process each task through $M$ ensemble members of frozen tabular foundation models, instantiated as \TabPFN-v3. For a given task $k$, we extract the related embeddings $\mathbf E_k = \{\ve_{mqk}\}_{m=1,q=1}^{M,Q} \in \mathbb R^{M\times Q\times p}$, where $Q$ and $p$ represent respectively the query point and the embedding dimensions.

We first map each query embedding into a common router latent space of dimension $h$. Specifically, we employ a shared encoder $g_\psi:\mathbb R^p\rightarrow\mathbb R^h$, which takes as inputs the embeddings $\ve_{mqk}$ for each ensemble member and query point and provides a latent representation

\begin{equation*}
    \vh_{mqk} = g_\psi(\ve_{mqk}), \qquad \vh_{mqk}\in\mathbb R^h.
    \label{eq:router_query_encoder}
\end{equation*}
More explicitly, $g_\psi$ is implemented as a two-layer Multilayer Perceptron (MLP) with Gaussian Error Linear Unit (GELU) activation function preceded by layer normalization (LN), such that
\begin{align*}
    \widetilde{\ve}_{mqk}
    &= \operatorname{LN}(\ve_{mqk}),
    \\
    \vh^{(1)}_{mqk}
    &= \operatorname{GELU} \left(W^{(1)}_\psi \widetilde{\ve}_{mqk} + b^{(1)}_\psi \right),
    \\
    \vh_{mqk}
    &= \operatorname{GELU} \left(W^{(2)}_\psi \vh^{(1)}_{mqk} + b^{(2)}_\psi\right).
    \label{eq:router_query_encoder_explicit}
\end{align*}

\paragraph{Query-set aggregation.} In this phase, we aggregate the $Q$ query-level representations $\vh_{mqk}$ associated with each
ensemble member using two different pooling schemes. Firstly, we employ a learned attention-pooling mechanism based on $H$ trainable attention queries $\{\va_j\}_{j=1}^H$. Each attention query defines a scoring direction in the router latent space and therefore induces a different weighting of the same set of query representations. Specifically, for each ensemble member $m$ and attention query $j$, we compute the normalized attention weights
\begin{equation*}
  \alpha_{mjqk} = \softmax_q\! \left( \frac{\va_j^\top\vh_{mqk}}{\sqrt{h}}\right), \qquad \sum_{q=1}^{Q}\alpha_{mjqk}=1.
\end{equation*}
which determine the contribution of each query representation to the pooled summary
\begin{equation}
    \vu_{mjk} = \sum_{q=1}^{Q} \alpha_{mjqk}\vh_{mqk}, \qquad \vu_{mjk}\in\mathbb R^h.
    \label{eq:router_attention_pool}
\end{equation}
Following this strategy, we are able to obtain $H$ distinct weighted summaries of the same query set, allowing the router to capture various aspects of the query-level representations.

We further enrich the learned attention summaries in \eqref{eq:router_attention_pool} with two standard aggregation over the query dimension $Q$, i.e., the element-wise average $\bar{\vh}_{mk}$ and element-wise standard deviation $\boldsymbol{\sigma}^{(Q)}_{mk}$, yielding a final summary of the query set
\begin{equation}
    \mathbf{c}_{mk}= \left[\bar{\vh}_{mk}^{\top}, \left(\boldsymbol{\sigma}^{(Q)}_{mk}\right)^{\top}, \vu_{m1k}^{\top}, \ldots, \vu_{mHk}^{\top} \right]^{\top} \in \mathbb R^{(H+2)h}.
    \label{eq:router_estimator_concat}
\end{equation}
Finally, we employ a shared estimator-level
projection $f_\psi:\mathbb R^{(H+2)h}\rightarrow\mathbb R^h$ represented by a two-layer MLP 
\begin{align}
    \widetilde{\vs}_{mk}
    &=
    \operatorname{GELU} \left(W^{(3)}_\psi \mathbf{c}_{mk} + b^{(3)}_\psi \right), \notag\\
    \vs_{mk}
    &= \operatorname{GELU} \left(W^{(4)}_\psi \widetilde{\vs}_{mk} + b^{(4)}_\psi \right),
    \label{eq:router_estimator_projection}
\end{align}
and obtain an $h$-dimensional representation of the previous summaries reported in \eqref{eq:router_estimator_concat}.

\paragraph{Task-level aggregation.} At this point, to obtain a single representation for task $k$, we again use the element-wise average $\bar{\vs}_k$ and standard deviation $\boldsymbol{\sigma}^{(M)}_k$ of the last output in \eqref{eq:router_estimator_projection} across the $M$ ensemble members, obtaining
\begin{equation}
    \vz_k = \begin{bmatrix}
        \bar{\vs}_k \\
        \boldsymbol{\sigma}^{(M)}_k
    \end{bmatrix} \in\mathbb R^{2h}.
    \label{eq:router_task_embedding}
\end{equation}
\paragraph{Routing probabilities.} The previous task-level aggregation in \eqref{eq:router_task_embedding} is then passed through a two-layer routing network $r_\psi:\mathbb R^{2h}\rightarrow\mathbb R^E$, where $E$ denotes the number of residual experts. In our setting, as previously mentioned, we have 3 residual experts which corresponds to Global Head, Expansion Specialist and Shrinkage Specialist. The routing network is
\begin{align*}
    \mathbf{r}_k
    &= \operatorname{GELU} \left(W_\psi^{(5)}\vz_k+b_\psi^{(5)}\right),\\
    \boldsymbol{\ell}_k
    &= W_\psi^{(6)}\mathbf{r}_k+b_\psi^{(6)}, \qquad \boldsymbol{\ell}_k\in\mathbb R^E.
    \label{eq:router_logits_explicit}
\end{align*}
The dense routing probabilities are then obtained through
\begin{equation}
    \pi_{ke} = \frac{\exp(\ell_{ke}/\tau)}{\displaystyle\sum_{e'=1}^{E} \exp(\ell_{ke'}/\tau) }, \qquad  e=1,\ldots,E,
    \label{eq:router_dense_probabilities}
\end{equation}
where $\tau>0$ is the routing temperature. We use $\tau=1$ in the
reported experiments.

\paragraph{Top-2 routing.} This final operation sparsifies the dense routing probabilities in
\eqref{eq:router_dense_probabilities} by retaining only the two experts with the largest routing probabilities. For task $k$, we denote the set of selected experts by $\mathcal A_k=\operatorname{Top2}(\boldsymbol{\pi}_k) \subseteq\{1,\ldots,E\}$. The remaining experts are masked out, and the probabilities associated with the selected experts are renormalized to obtain the sparse routing weights
\begin{equation}
    w_{ke} = \frac{\pi_{ke}\, \mathbbm{1}\{e\in\mathcal A_k\}}{\displaystyle \sum_{e'\in\mathcal A_k} \pi_{ke'}}, \qquad e=1,\ldots,E.
    \label{eq:router_sparse_weights}
\end{equation}
This construction forces the router to concentrate its decision on a small subset of experts, rather than distributing probability mass over the entire mixture, while preserving a valid convex combination of the selected expert distributions. Consequently,
\begin{equation*}
    w_{ke}=0 \quad\text{for }e\notin\mathcal A_k, \qquad \sum_{e=1}^{E} w_{ke}=1, \qquad \|{\mathbf w}_k\|_0=2.
    \label{eq:router_sparse_properties}
\end{equation*}
In our setting, where $E=3$, Top-2 routing therefore assigns zero weight to the expert with the smallest dense routing probability and combines the remaining two experts.

\subsection{Post-training Reward Function}\label{app:pt_details}

As described in Section~\ref{sec:method}, the post-training stage refines the router by evaluating the empirical properties of the confidence intervals induced by candidate routing actions sampled around the current routing policy. Specifically, we design the reward on the basis of the empirical coverage of these intervals while simultaneously controlling their length. Balancing these two quantities is essential, since rewarding coverage alone may encourage excessively wide intervals leading to overcoverage, whereas encouraging short intervals too strongly may lead to undercoverage. We thus decompose the total reward $r_{\mathrm{total}}$ into six components, namely $r_{\mathrm{WIL}}$, $r_{\mathrm{WIS}}$, $r_{\mathrm{CS}}$, $r_{\mathrm{tail}}$, $r_{\mathrm{CVaR}}$, and $r_{\mathrm{excess}}$, which control the weighted interval length (WIL), the weighted interval score (WIS), the coverage score (CS), the upper and lower tails (tail), the undercoverage (CVaR), and the excess width (excess), respectively. The total reward is defined as
\begin{gather}
    r_{total} = \gamma_{1}\,r_{WIL} + \gamma_{2}\,r_{WIS} - \gamma_{3}\,r_{CS} - \gamma_{4}\,r_{tail} - \gamma_{5}\,r_{CVaR} - \gamma_{6}\,r_{excess}, \label{eqn:reward}
\end{gather}
where the coefficients $\gamma_{l} \ge 0$, for $l \in \{1,2,3,4,5,6\}$, control the contribution of each component. We describe each component in detail below.

\paragraph{Weighted interval length.} To encourage sharper confidence intervals, we consider their weighted average length across the nominal coverage levels that belong to a set $\mathcal C$. For a given task $k$, we define
\begin{equation}
    WIL_k = \sum_{\ell=1}^{|\mathcal C|} \nu_\ell \left\{\frac{1}{Q} \sum_{q=1}^{Q} \left[U_{k\ell q}-L_{k\ell q} \right] \right\},
    \label{eq:wil}
\end{equation}
where $\nu_\ell$ is the weight associated with nominal coverage level $c_\ell$, and $U_{k\ell q}$ and $L_{k\ell q}$ denote the confidence interval bounds for query point $q$.

We evaluate \eqref{eq:wil} under two residual distributions for the same task. First, $WIL_k^{(G)}$ is obtained by constructing the confidence intervals defined in \eqref{eqn:proposedCI} from the Global residual distribution $p_G$, i.e., using the quantiles of the residual distribution produced by the Global Head alone. Second, for a routing action $a$, $WIL_k^{(a)}$ is computed from the confidence intervals induced by the corresponding sparse MoE residual distribution $p(a)$. In both cases, the weighted interval length is evaluated over the same query points and nominal coverage levels, so that their difference isolates the change in interval width introduced by the routing action. We then define the WIL reward as
\begin{equation*}
    r_{\mathrm{WIL},k}(a) = \frac{WIL_k^{(G)}-WIL_k^{(a)}}{|WIL_k^{(G)}|}.
    \label{eq:reward_wil}
\end{equation*}
Therefore, the Global Head provides a task-specific reference against which each routing action is evaluated. Positive values are related to shorter intervals than those produced by the Global Head, whereas negative values indicate wider intervals.

\paragraph{Weighted interval score.} While the WIL only measures interval width, the weighted interval score additionally accounts for whether the true regression function lies within the confidence interval. For a nominal coverage level $c_\ell\in\mathcal C$, let $\alpha_\ell=1-c_\ell$. Given an interval $[L,U]$ for an oracle regression function $z= f(x)$, its interval score is defined as
\begin{equation*}
    \mathrm{IS}_{\alpha}(L,U;z) = (U-L) + \frac{2}{\alpha}(L-z)\mathbbm{1}\{z<L\} + \frac{2}{\alpha}(z-U)\mathbbm{1}\{z>U\}.
    \label{eq:is}
\end{equation*}
The score thus increases both with interval width and when the target falls outside the interval. Since the latter penalties are scaled by $2/\alpha$, coverage failures are penalized more heavily at higher nominal coverage levels.

Aggregating the interval scores across the coverage levels and query points, we define the WIS for task $k$ as
\begin{equation}
    WIS_k = \frac{1}{Q\left(|\mathcal C|+\frac{1}{2}\right)} \sum_{q=1}^{Q} \left[\frac{1}{2}|z_{kq}-m_{kq}| + \sum_{\ell=1}^{|\mathcal C|} \frac{\alpha_\ell}{2} \mathrm{IS}_{\alpha_\ell} \left(L_{k\ell q}, U_{k\ell q}; z_{kq} \right) \right],
    \label{eq:wis}
\end{equation}
where $m_{kq}$ denotes the median of the residual distribution under consideration and $z_{kq}$ is the oracle regression function for query point $q$.

Similarly to WIL, we evaluate \eqref{eq:wis} under both the Global and routed residual distributions. In particular, $WIS_k^{(G)}$ is obtained using the median and confidence intervals derived from the Global residual distribution $p_G$, whereas $WIS_k^{(a)}$ uses those induced by the sparse MoE residual distribution associated with routing action $a$. Since both quantities are evaluated on the same task, query points, and nominal coverage levels, we define the reward as
\begin{equation*}
    r_{\mathrm{WIS},k}(a) = \frac{WIS_k^{(G)}-WIS_k^{(a)}}{|WIS_k^{(G)}|}.
    \label{eq:reward_wis}
\end{equation*}
Hence, positive values encourage routing actions that improve the interval score with respect to the Global Head, whereas negative values indicate a deterioration.

\paragraph{Coverage score.} While WIL and WIS prefer sharper and better-scoring intervals, they do not explicitly enforce calibration. Accordingly, we introduce a coverage component that penalizes routing actions whose empirical coverage deviates from the nominal level beyond a fixed tolerance. For task $k$, routing action $a$, and nominal coverage level $c_\ell\in\mathcal C$, we compute
\begin{equation*}
    \widehat{CP}_{k\ell}^{(a)} = \frac{1}{Q} \sum_{q=1}^{Q} \mathbbm{1} \left\{L_{k\ell q}^{(a)} \le f_k(x_{kq}) \le U_{k\ell q}^{(a)} \right\},
    \label{eq:empirical_coverage}
\end{equation*}
where $f_k(x_{kq})$ denotes the true regression function evaluated at query point $x_{kq}$. Rather than penalizing any finite-sample deviation from $c_\ell$, we allow a level-specific tolerance $\epsilon_\ell$ and define the under- and overcoverage violations as
\begin{equation*}
    d_{k\ell}^{-,(a)} =\left[c_\ell-\epsilon_\ell-\widehat{CP}_{k\ell}^{(a)}\right]_+,
    \qquad d_{k\ell}^{+,(a)} =\left[\widehat{CP}_{k\ell}^{(a)}-c_\ell-\epsilon_\ell \right]_+,
    \label{eq:coverage_violations}
\end{equation*}
where $[x]_+=\max(x,0)$. Hence, no coverage penalty is incurred as long as the empirical coverage remains within the interval
$[c_\ell-\epsilon_\ell,c_\ell+\epsilon_\ell]$.

The violations are weighted by adaptive dual variables, allowing the optimization to place greater emphasis on coverage levels for which calibration remains difficult. In addition, the penalty combines a global constraint, shared across all tasks, with a complexity-specific constraint associated with the bucket $b_k$ of task $k$. We define
\begin{align*}
    r_{\mathrm{CS},k}(a)
    &=
    \omega_G \sum_{\ell=1}^{|\mathcal C|} \left[\lambda_{G\ell}^{-} d_{k\ell}^{-,(a)} + \lambda_{G\ell}^{+} d_{k\ell}^{+,(a)}\right]
    \nonumber\\
    &\quad+
    (1-\omega_G) \sum_{\ell=1}^{|\mathcal C|} \left[\lambda_{b_k\ell}^{-} d_{k\ell}^{-,(a)} + \lambda_{b_k\ell}^{+} d_{k\ell}^{+,(a)}\right],
    \label{eq:coverage_score}
\end{align*}
where $\lambda^{-}$ and $\lambda^{+}$ denote the dual multipliers associated with under- and overcoverage, respectively. We set $\omega_G=0.5$, assigning equal weight to the global and complexity-specific coverage constraints.

\paragraph{Tail balance.}

Controlling the overall coverage does not ensure that the errors are distributed across the two tails. Indeed, an interval may attain coverage close to the nominal level while most of the missed values fall below the lower bound or above the upper bound. We therefore separately control the empirical lower- and upper-tail error probabilities. For a nominal coverage level $c_\ell$, a task $k$ and routing action $a$ and $\alpha_\ell=1-c_\ell$, we define
\begin{equation*}
    \widehat{T}_{k\ell}^{-,(a)} = \frac{1}{Q} \sum_{q=1}^{Q} \mathbbm{1} \left\{f_k(x_{kq})<L_{k\ell q}^{(a)} \right\}, \qquad \widehat{T}_{k\ell}^{+,(a)} = \frac{1}{Q} \sum_{q=1}^{Q} \mathbbm{1} \left\{f_k(x_{kq})>U_{k\ell q}^{(a)}\right\}.
    \label{eq:tail_probabilities}
\end{equation*}
For an equal-tailed confidence interval, the nominal error probability assigned to each tail is $\alpha_\ell/2$. Allowing for a level-specific tolerance $\tau_\ell$, we measure the corresponding deviations as
\begin{equation*}
    d_{k\ell}^{-,(a)} = \widehat{T}_{k\ell}^{-,(a)} - \frac{\alpha_\ell}{2} - \tau_\ell, \qquad d_{k\ell}^{+,(a)} = \widehat{T}_{k\ell}^{+,(a)} - \frac{\alpha_\ell}{2} - \tau_\ell.
    \label{eq:tail_violations}
\end{equation*}
Consequently, positive values indicate that the routing action assigns excessive error probability to the tail. Similarly to the coverage component, these deviations are weighted by adaptive dual variables at both the global and complexity-specific levels, yielding
\begin{equation*}
    r_{\mathrm{tail},k}(a) = \sum_{\ell=1}^{|\mathcal C|} \left[\lambda_{G\ell}^{-}d_{k\ell}^{-,(a)} + \lambda_{G\ell}^{+}d_{k\ell}^{+,(a)} + \lambda_{b_k\ell}^{-}d_{k\ell}^{-,(a)} + \lambda_{b_k\ell}^{+}d_{k\ell}^{+,(a)} \right].
    \label{eq:tail_balance}
\end{equation*}
This component discourages routing actions for which the miscoverage is concentrated disproportionately in either tail, alongside the control of overall coverage.

\paragraph{CVaR.} The previous coverage component controls deviations from the nominal coverage level, but does not explicitly account for tasks characterized by severe undercoverage. To control these cases, we introduce a Conditional Value-at-Risk (CVaR) penalty based on the undercoverage violation defined above. Specifically, for task $k$, routing action $a$, and nominal coverage level $c_\ell$, we consider
\begin{equation*}
    d_{k\ell}^{(a)} = \left[c_\ell-\epsilon_\ell-\widehat{CP}_{k\ell}^{(a)} \right]_+,
    \label{eq:cvar_deficit}
\end{equation*}
which is positive only when the empirical coverage falls below the tolerated nominal level.

Rather than penalizing these violations only through their average behaviour, CVaR places focus on the upper tail of their distribution across tasks. For a given scope $s$, either global or complexity-specific, we use the standard CVaR representation
\begin{equation*}
    C_{s,k\ell}^{(a)} = \eta_{s\ell} + \frac{1}{1-\beta} \left[d_{k\ell}^{(a)}-\eta_{s\ell}\right]_+,
    \label{eq:cvar_surrogate}
\end{equation*}
where $\eta_{s\ell}$ is an adaptive threshold and $\beta$ corresponds to the fraction of severe undercoverage events receiving more importance. We set $\beta=0.80$ and update $\eta_{s\ell}$ throughout post-training.

Similarly to the coverage score, we combine the global and complexity-specific components and aggregate them across nominal coverage levels,
\begin{equation*}
    r_{\mathrm{CVaR},k}(a) = \sum_{\ell=1}^{|\mathcal C|} \nu_\ell \left[\omega_G C_{G,k\ell}^{(a)} + (1-\omega_G)C_{b_k,k\ell}^{(a)}\right],
    \label{eq:cvar_penalty}
\end{equation*}
where $\omega_G=0.5$. This penalty places additional weight on routing actions associated with the most pronounced undercoverage, preventing good average calibration from masking poorly calibrated tasks.

\paragraph{Excess width.} Although WIL encourages shorter intervals, some widening may be necessary to correct undercoverage, especially for more complex tasks. We thus introduce an excess-width penalty that discourages unnecessary widening of the Global Head while allowing a complexity-specific margin. For task $k$ and routing action $a$, we first compute the weighted interval length
\begin{equation*}
    R_k^{(a)} = \sum_{\ell=1}^{|\mathcal C|}
    \nu_\ell \frac{\frac{1}{Q}\sum_{q=1}^{Q} \left[U_{k\ell q}^{(a)}-L_{k\ell q}^{(a)}\right]}{\frac{1}{Q}\sum_{q=1}^{Q} \left[U_{k\ell q}^{(G)}-L_{k\ell q}^{(G)}\right]}.
    \label{eq:relative_width}
\end{equation*}
The excess width is then measured against a complexity-specific budget $b_k$ as
\begin{equation*}
    e_k^{(a)}=\left[R_k^{(a)}-1-b_k \right]_+.
    \label{eq:excess_width}
\end{equation*}
We use $b_k\in\{0,0.03,0.25\}$ for low-, medium-, and high-complexity tasks, respectively, allowing progressively more widening as task complexity increases. Since widening may be beneficial when coverage is insufficient, the resulting penalty is downweighted when the routing action exhibits undercoverage, while remaining active once coverage becomes adequate. The final excess-width component is then defined as
\begin{equation*}
    r_{\mathrm{excess},k}(a) = e_k^{(a)}\,\rho_k^{(a)},
    \label{eq:reward_excess}
\end{equation*}
where $\rho_k^{(a)}$ is a coverage-dependent weight that reduces the penalty when additional width is needed to recover coverage. In our implementation, this weight is bounded below by $0.25$, preventing excessive widening from becoming entirely unpenalized.

\section{Implementation Details of the Experiments}\label{app:impl_details}

This section provides additional implementation details for the experimental evaluation described in Section~\ref{sec:experiment}. We describe the construction of the synthetic and semi-synthetic evaluation tasks, including the OpenML-\ul{\textit{CC18}} benchmark, the training and implementation configuration of \TabCon, and the evaluation metrics used throughout the experiments.

\subsection{Synthetic DGPs}\label{app:dgps}

As described in Section~\ref{sec:experiment}, the synthetic evaluation contains both in-distribution and out-of-distribution DGPs, together with a separate stress-test benchmark. All evaluation datasets are generated independently from the datasets used for training, including those associated with the in-distribution DGPs.

\paragraph{Standard synthetic benchmark.} The standard synthetic benchmark evaluates \TabCon\ over a broad distribution of regression tasks. In the in-distribution setting, the DGP components are sampled from the same general family used during training, allowing the regression function, observation-noise mechanism, covariate structure, dimensionality, and context size to vary jointly across tasks. The resulting benchmark therefore reflects the natural heterogeneity induced by the synthetic prior. The number of context observations ranges from $8$ to $2{,}048$, while the number of covariates varies from $1$ to $160$; the number of query points is sampled from $\{32,64,128,256\}$.

The out-of-distribution DGPs are instead based on functional families that are not observed during training. In particular, we employ Random Fourier Feature (RFF) functions, multilayer perceptrons (MLPs), and low-rank Gaussian-process (GP) draws. These tasks are generated with context sizes in $\{32,64,128,256\}$, numbers of covariates in $\{4,8,16,32\}$, and $64$ query points per task. Observation noise is independently varied across Gaussian, Student-$t$, heteroskedastic, and contaminated distributions. In both the in- and out-of-distribution settings, the regression function is generated explicitly and thus remains known at every queried covariate value, allowing us to evaluate confidence-interval coverage directly against the ground truth.

\paragraph{Stress-test benchmark.} 
The stress-test benchmark serves a different purpose. Rather than sampling all DGP components jointly from the standard task distribution (as in standard synthetic benchmark), each setting introduces a specific component (e.g.,regression-function, noise, and feature configurations) while keeping the remaining components close to a common reference configuration. This controlled construction allows us to isolate the effect of individual sources of complexity that may otherwise be masked by averaging over heterogeneous tasks in the standard benchmark. We next explain the settings considered for each component.

For the regression function, we focus on (i) \textit{Bounded functions}, obtained by applying a hyperbolic-tangent transformation;
(ii) \textit{Quantized functions}, whose output is restricted to between $4$ and $64$ discrete levels; and (iii) \textit{Discontinuous functions}, which take between $2$ and $8$ constant values over different regions of the covariate space, inducing abrupt jumps between regions. For the noise mechanism, we test (i) \textit{Near-zero noise}, where the observation-noise scale is reduced to approximately $10^{-4}$--$10^{-2}$ times the reference signal scale; (ii) \textit{High noise}, where the noise scale is increased to approximately $0.5$--$2$ times the reference signal scale; (iii) \textit{Heavy-tailed noise}, generated from a Student-$t$ distribution with $3$--$8$ degrees of freedom; (iv) \textit{Heteroskedastic noise}, where the noise scale varies with the observed covariates; and (v) \textit{Asymmetric-outlier contamination}, where observations are contaminated on one side with probability between $0.01$ and $0.08$, with contaminated values shifted by approximately $10$--$50$ times the reference scale. 

Finally, we consider a \textit{categorical-plus-missing covariates} setting, which stresses the observed feature space rather than the regression function or noise mechanism. A subset of the covariates is converted into categorical variables with $2$--$12$ levels, while $10\%$--$30\%$ of the feature entries are set to missing according to either a missing-completely-at-random (MCAR) or missing-at-random (MAR) mechanism.

\subsection{Semi-synthetic CC18 Benchmark}\label{app:cc18}

To extend the fully synthetic experiments to realistic covariate distributions, we construct a semi-synthetic benchmark from $13$ datasets in the OpenML-\ul{\textit{CC18}} suite. Specifically, we consider \textit{MiceProtein}, \textit{analcatdata\_dmft}, \textit{bank-marketing}, \textit{blood-transfusion-service-center}, \textit{breast-w}, \textit{connect-4}, \textit{credit-approval},
\textit{first-order-theorem-proving}, \textit{jungle\_\allowbreak chess\_\allowbreak 2pcs\_\allowbreak raw\_\allowbreak endgame\_\allowbreak complete}, \textit{kc2}, \textit{mfeat-fourier}, \textit{pc4}, and \textit{phoneme}.

For each dataset, we preserve the original covariate distribution while discarding the observed outcome labels. This allows us to retain the real-world structure of the feature space while replacing the unknown data-generating mechanism with a designed model
\begin{equation}
    Y_i = f(X_i) + \varepsilon_i,
    \label{eq:cc18_response}
\end{equation}
where $f$ is a \textit{synthetic} regression function. 

Similarly to the synthetic benchmarks, we vary both the regression function and the observation-noise mechanism across tasks. In particular, the response-generating process reuses the same target and noise regimes considered by the synthetic generator, while replacing synthetic covariates with the real feature distributions of the \ul{\textit{CC18}} datasets. The base regression mapping is sampled from the four families (i.e., a smooth nonlinear transformation, an RFF-based mapping, a low-rank GP draw, and a tree-based mapping) and is further varied through regular, bounded, discontinuous, quantized, near-constant, and heavy-tailed target regimes. The observation noise is independently sampled from Gaussian, near-zero, high-noise, Student-$t$, skewed, asymmetric-contamination, and heteroskedastic regimes.

\subsection{Training Configuration}

The training configuration follows the pretraining and post-training pipeline introduced in Section~\ref{sec:method}. To construct the confidence intervals in \eqref{eqn:proposedCI}, we use the frozen \TabPFN-v3 backbone to obtain the conditional-mean estimate $\widehat f(x)$, while the pretraining and post-training stages are used to learn the residual distribution and estimate the quantiles required to form the lower and upper confidence bounds.

Each stage is trained on independently generated synthetic tasks sampled from the training DGP prior, with no reuse of datasets employed in the subsequent evaluation. In the following, we report the architecture and optimization settings of the residual heads and router, together with stage-specific hyperparameters used throughout the pipeline.

\paragraph{Global Head.}
We generate $50{,}000$ synthetic datasets from the full collection of training DGPs and train the Global Head to model the distribution of the estimation error. The head is implemented as a lightweight two-layer MLP applied to the final-layer query representation returned by the frozen \TabPFN-v3 backbone. Specifically, it maps the $512$-dimensional representation to $1{,}024$ hidden units through a GELU activation and then to $5{,}000$ logits, which parameterize the discretized residual distribution. Following \eqref{eqn:negative}, we minimize
\begin{equation*}
    \ell_{\mathrm{global}}
    = \ell_{\mathrm{NLL}} + 0.20\,\ell_{\mathrm{CRPS}}.
    \label{eq:global_loss_impl}
\end{equation*}
Only the parameters of the residual head are optimized. We use AdamW with learning rate $3\times10^{-4}$, weight decay $10^{-4}$, gradient accumulation over $8$ tasks, and gradient clipping at norm $1$.

\paragraph{Specialist Heads.}
After training the Global Head, we use its empirical coverage behaviour to assign each training DGP to one of three regimes: undercoverage, intermediate, and overcoverage. This assignment is based on the weighted signed coverage score introduced in Section~\ref{sec:method} and determines how the two Specialist Heads are updated on each task. We compute the weighted signed coverage score the Global Head at nominal coverage levels $50\%$, $80\%$, $90\%$, and $95\%$, assigning weights $0.25$, $0.50$, $1.00$, and $2.00$, respectively, so that higher coverage levels are weighted more heavily. 

Both the Expansion and Shrinkage Specialists retain the same residual-head architecture as the Global Head and are initialized from its trained parameters. For a task associated with specialist $e\in\{\mathrm{expansion},\mathrm{shrinkage}\}$, the corresponding head is optimized according to
\begin{equation*}
    \ell_{\mathrm{specialist}}=\ell_{\mathrm{NLL}} +0.20 \, \ell_{\mathrm{CRPS}}+ 0.005\, \textrm{KL}(p_{\textrm{global}},p_{\textrm{specialist}}).
    \label{eq:specialist_loss_impl}
\end{equation*}
The opposite specialist is updated only through the KL term, which limits deviations from the Global Head outside its target regime.

The specialist heads are optimized with AdamW using learning rate
$3\times10^{-5}$, weight decay $10^{-4}$, gradient accumulation over $4$ tasks, and gradient clipping at norm $1$. We use a linear warm-up over the first $5\%$ of optimization updates followed by cosine learning-rate decay.

\paragraph{Router configuration.}
The router takes exclusively the query embeddings returned by the frozen \TabPFN ensemble members as input. Following the architecture described in Appendix~\ref{app:router_details}, we use $M=2$ ensemble members, set both the query-level and ensemble-member latent dimensions to $h=96$, and employ $H=4$ learned attention queries. The attention queries are initialized from $\mathcal N(0,h^{-1})$, while the weights and bias of the final routing layer are initialized to zero.

\paragraph{Router initialization through SFT.}
We generate $50{,}000$ additional synthetic datasets from the full collection of training DGPs and use them to train the router through supervised fine-tuning (SFT) minimizing

\begin{eqnarray*}\label{eqn:SFTloss-conf}
\ell_{\textrm{SFT}} = \ell_{\textrm{NLL}} + \ell_{\textrm{CRPS}},
\end{eqnarray*}
where $\eta_3=1$ in \eqref{eqn:SFTloss}. The loss $\ell_{\textrm{SFT}}$ is evaluated on the final residual mixture distribution obtained by combining the three frozen residual heads according to the task-dependent router weights, focusing only on the optimization of the router parameters.

We use AdamW with learning rate $10^{-4}$, weight decay $10^{-4}$, gradient accumulation over $4$ tasks, and gradient clipping at norm $1$. The learning rate is linearly warmed up over the first $5\%$ of optimization updates and subsequently follows a cosine decay schedule. During training, the number of queried covariate values cycles over $Q\in\{64,128,256,256\}$, assigning greater frequency to the largest query
set.

\paragraph{RL-based post-training.}
Starting from the SFT initialization, we generate $100{,}000$ additional synthetic datasets and further optimize the router using RLOO. For each task, the router defines a policy over routing actions, with every sampled action projected through the same strict Top-$2$ rule used at inference before constructing the related residual mixture. The exploration standard deviation is linearly annealed from $1.0$ to $0.10$ over the course of post-training. Following RLOO, the reward of each action is compared against the average reward of the remaining candidates, which serves as a leave-one-out baseline.

The reward is computed at the task level by aggregating coverage and interval properties across the queried covariate values, as detailed in Appendix~\ref{app:pt_details}. We set the coefficients in \eqref{eqn:reward} to $\gamma_{\mathrm{WIL}}=1, \gamma_{\mathrm{WIS}}=0.5, \gamma_{\mathrm{CS}}=1, \gamma_{\mathrm{tail}}=1, \gamma_{\mathrm{CVaR}}=1, \gamma_{\mathrm{e}}=1$. For the WIL and WIS components, we consider the nominal coverage levels $\mathcal{C}= \{0.50,0.80,0.90,0.95\}$ with the corresponding weights $\boldsymbol{\nu} =(0.10,\,0.20,\,0.30,\,0.40)$, leading to place a stronger contribution for higher coverage levels.

The router is optimized with AdamW using a learning rate of $10^{-4}$, weight decay $10^{-4}$, gradient accumulation over $4$ tasks, and gradient clipping at norm $1$. We retain the same linear warm-up and cosine-decay schedule used during SFT. 

\subsection{Evaluation Metrics}\label{app:metrics}

In our experiments, we mainly evaluate the confidence intervals produced by \TabCon\ in terms of empirical coverage probability (CP) and interval length (IL). Together, these metrics summarize the calibration and sharpness of the resulting confidence intervals, with CP measuring agreement with the nominal coverage level and IL quantifying interval width.

We first consider the empirical coverage. For each evaluation task $k$ and query point $x_{kq}$, we define the interval $[L_{kq},U_{kq}]$ returned by a given method. Its empirical coverage is computed by assessing whether the evaluation target falls within the reported interval, as
\begin{equation*}
    \mathrm{CP} = \frac{1}{N_{\mathrm{eval}}} \sum_{k,q} \mathbbm{1} \left\{L_{kq} \leq T_{kq} \leq U_{kq} \right\},
    \label{eq:coverage_metric}
\end{equation*}
where $T_{kq}$ denotes the target associated with the interval returned by the corresponding method, namely the regression function value for confidence intervals and the realized response for prediction intervals.

We next define the empirical interval length
\begin{equation*}
    \mathrm{IL} = \frac{1}{N_{\mathrm{eval}}} \sum_{k,q} \left( U_{kq}-L_{kq} \right).
    \label{eq:interval_length_metric}
\end{equation*}
Since the interval width is meaningful only compared with the achieved coverage, we assess CP and IL jointly when comparing the resulting intervals across methods.

\section{Additional Empirical Results}\label{app:add_emp_results}

This section extends the empirical analysis in Section~\ref{sec:experiment} along three dimensions. First, we provide a more granular view of the results at the $95\%$ nominal coverage level by reporting CP and IL separately across the \ul{\textit{CC18}} datasets and individual stress-test DGPs. This allows us to examine the heterogeneity that is hidden by the aggregate results and to identify the settings in which calibration is more challenging. Second, we evaluate all methods at additional nominal coverage levels of $90\%$ and $99\%$. For both levels, we report aggregate results across the four evaluation settings (i.e., Standard (ID), Standard (OOD), Stress-test, CC18) together with \ul{\textit{CC18}} datasets and individual stress-test DGPs analysis, allowing us to assess whether the calibration--sharpness behavior observed at $95\%$ persists across different levels of uncertainty. Finally, we provide a more detailed analysis of computational efficiency. In addition to the end-to-end inference times reported in Section~\ref{sec:experiment}, we separate the common \TabPFN\ backbone computation from the method-specific overhead of the uncertainty construction, thereby isolating the additional cost introduced by the residual heads and task-level routing.

\subsection{Additional Results at the 95\% Nominal Level}\label{app:per_dgp} 

\paragraph{Semi-Synthetic CC18 Benchmark.} We report coverage probability and interval length at 95\% nominal level separately for each dataset \ul{\textit{CC18}} across the baseline algorithms (see Section \ref{sec:experiment}) and \TabCon\ in Table~\ref{tab:cc18-per-dataset-95}. 

\begin{table}[!htbp]
  \centering
  \caption{Per-Dataset coverage probability (CP) and interval length (IL) on the CC18 real-covariate semi-synthetic benchmark at the 95\% nominal level. From left to right, results are reported for Linear, RFF, PI, Bootstrap, AMP and \TabCon. For CP, the value closest to the nominal level of 0.95 is shown in \textbf{bold}, and the second closest is \ul{underlined}. For IL, among methods satisfying $\mathrm{CP}\geq 0.95$, the shortest IL is shown in \textbf{bold}, and the second shortest is \ul{underlined}.}
  \label{tab:cc18-per-dataset-95}
  \setlength{\tabcolsep}{2.1pt}
  \renewcommand{\arraystretch}{1.10}
  \fontfamily{ptm}\fontsize{8pt}{9.6pt}\selectfont
  \begin{tabular*}{\textwidth}{@{\extracolsep{\fill}}lcccccc@{}}
    \toprule
    Dataset & Linear & RFF & PI & Bootstrap & AMP & \TabCon \\
    & \multicolumn{1}{c}{CP / IL} & \multicolumn{1}{c}{CP / IL} & \multicolumn{1}{c}{CP / IL} & \multicolumn{1}{c}{CP / IL} & \multicolumn{1}{c}{CP / IL} & \multicolumn{1}{c}{CP / IL} \\
    \midrule
    MiceProtein & 0.6978 / 1.0202 & 0.6438 / 0.9617 & 0.9864 / \textbf{2.6205} & 0.6025 / 0.3484 & \underline{0.9781} / \underline{2.6365} & \textbf{0.9269} / 1.0352 \\
    analcatdata-dmft & 0.8121 / 0.5280 & 0.8822 / 1.0762 & 0.9992 / 2.3870 & 0.8497 / 0.6667 & \underline{0.9891} / \underline{2.0333} & \textbf{0.9857} / \textbf{0.8547} \\
    bank-marketing & 0.5060 / 0.1667 & 0.4452 / 0.2022 & 0.9974 / 4.0046 & 0.7074 / 0.4495 & \textbf{0.9799} / \underline{3.2497} & \underline{0.9942} / \textbf{2.1687} \\
    blood-transfusion & 0.3367 / 0.1761 & 0.6833 / 0.4412 & 0.9964 / 1.1264 & 0.8120 / 0.2576 & \textbf{0.9746} / \underline{1.0751} & \underline{0.9777} / \textbf{0.2611} \\
    breast-w & 0.4179 / 0.5153 & 0.7670 / 1.5191 & 0.9943 / 3.0226 & 0.7420 / 0.6842 & \textbf{0.9628} / \underline{2.5238} & \underline{0.9728} / \textbf{1.6167} \\
    connect-4 & 0.5950 / 0.0538 & 0.3602 / 0.0564 & 0.9959 / 1.3124 & 0.7768 / 0.1558 & \textbf{0.9910} / \underline{1.3120} & \underline{0.9923} / \textbf{0.8986} \\
    credit-approval & 0.7570 / 0.6141 & 0.8058 / 0.9372 & 0.9923 / \underline{2.1683} & 0.7072 / 0.2119 & \underline{0.9758} / \textbf{2.0549} & \textbf{0.9420} / 0.5922 \\
    first-order-theorem & 0.5015 / 0.2439 & 0.4858 / 0.2398 & 0.9938 / \underline{1.5247} & 0.6868 / 0.3097 & \underline{0.9832} / 1.5520 & \textbf{0.9708} / \textbf{0.9619} \\
    jungle-chess & 0.1863 / 0.0554 & 0.3986 / 0.1388 & 0.9998 / 3.7744 & 0.8460 / 0.4883 & \textbf{0.9902} / \underline{2.9580} & \underline{0.9921} / \textbf{1.2117} \\
    kc2 & 0.7587 / 1008.9809 & 0.8923 / 1.9816 & 0.9942 / 2.5643 & 0.7698 / 0.4573 & \underline{0.9904} / \underline{2.5366} & \textbf{0.9566} / \textbf{0.8106} \\
    mfeat-fourier & 0.6955 / 0.9144 & 0.6462 / 0.9106 & 0.9963 / \underline{4.2712} & 0.6217 / 0.3190 & \underline{0.9863} / \textbf{3.9521} & \textbf{0.9225} / 1.0230 \\
    pc4 & 0.5208 / 1.0582 & 0.6317 / 1.0370 & 0.9940 / \textbf{3.4777} & 0.6394 / 0.5831 & \underline{0.9870} / \underline{3.5839} & \textbf{0.9402} / 2.0009 \\
    phoneme & 0.2048 / 0.1177 & 0.4677 / 0.2656 & 0.9993 / 2.2362 & 0.7654 / 0.3776 & \textbf{0.9747} / \underline{1.8834} & \underline{0.9771} / \textbf{0.4564} \\
    \bottomrule
  \end{tabular*}
\end{table}

Overall, \TabCon\ attains a near-coverage at 95\% in 7 out of the 13 CC18 datasets and the shortest interval length in 9 out of 13 when the coverage level reaches at least 95\%. Looking across individual datasets, \TabCon\ mildly undercovers in four cases, namely \textit{MiceProtein}, \textit{credit-approval}, \textit{mfeat-fourier}, and \textit{pc4}, with coverage remaining between approximately 0.92 and 0.94. Interestingly, these datasets tend to exhibit complex covariate structures. \textit{MiceProtein} and \textit{mfeat-fourier} contain high-dimensional continuous representations, \textit{credit-approval} combines numerical and categorical covariates together with missing observations, and \textit{pc4} contains a rich collection of software-complexity metrics. This pattern suggests that the remaining undercoverage may be associated with the interaction between dimensionality and the geometry of the covariate distribution, which can lead \TabCon\ to produce intervals that are slightly too narrow. However, dimensionality alone does not explain this behavior. For example, \textit{first-order-theorem} and \textit{connect-4} are also high-dimensional, yet \TabCon\ mildly overcovers on both datasets.

On the other side, AMP attains a near-coverage at 95\% in 6 out of the 13 CC18 datasets. However, AMP systematically tends to overcover and produces wider intervals than \TabCon. Overall, these results highlight the sharper intervals achieved by \TabCon\ while maintaining coverage close to the nominal level across heterogeneous real-covariate distributions. 

PI shows a more conservative behavior, producing coverage close to one across the CC18 datasets, but at the cost of longer intervals. This is consistent with PI being constructed from TabPFN's predictive distribution, which captures not only uncertainty in the estimated regression function but also outcome variability. In contrast, Bootstrap, Linear, and RFF undercover across all datasets, indicating that their shorter intervals do not provide adequate coverage. Additionally, Linear expresses an extremely large interval length on \textit{kc2}, suggesting instability in this covariate distribution.

\paragraph{Synthetic Benchmark.} We report coverage probability and interval length at 95\% nominal level separately for each stress-test DGP across the baseline algorithms and \TabCon\ in Table~\ref{tab:stress-per-regime-95}. 

\begin{table}[!htbp]
  \centering
  \caption{Per-DGP stress-test coverage probability (CP) and interval length (IL) at the 95\% nominal level. From left to right, results are reported for Linear, RFF, PI, Bootstrap, AMP and \TabCon. For CP, the value closest to the nominal level of 0.95 is shown in \textbf{bold}, and the second closest is \ul{underlined}. For IL, among methods satisfying $\mathrm{CP}\geq 0.95$, the shortest IL is shown in \textbf{bold}, and the second shortest is \ul{underlined}.}
  \label{tab:stress-per-regime-95}
  \setlength{\tabcolsep}{2.5pt}
  \renewcommand{\arraystretch}{1.12}
  \fontfamily{ptm}\fontsize{8pt}{9.6pt}\selectfont
  \begin{tabular*}{\textwidth}{@{\extracolsep{\fill}}lcccccc@{}}
    \toprule
    DGP & Linear & RFF & PI & Bootstrap & AMP & \TabCon \\
    & \multicolumn{1}{c}{CP / IL} & \multicolumn{1}{c}{CP / IL} & \multicolumn{1}{c}{CP / IL} & \multicolumn{1}{c}{CP / IL} & \multicolumn{1}{c}{CP / IL} & \multicolumn{1}{c}{CP / IL} \\
    \midrule
    Regular control & 0.7584 / 3.6781 & 0.7127 / 2.0034 & 0.9898 / \textbf{1.8534} & 0.7115 / 0.4684 & \underline{0.9772} / \underline{1.8968} & \textbf{0.9482} / 1.2188 \\
    Near-constant & 0.7470 / 0.0359 & 0.7252 / 0.0207 & 0.9937 / 0.0199 & 0.7167 / 0.0043 & \underline{0.9821} / \underline{0.0196} & \textbf{0.9556} / \textbf{0.0110} \\
    Bounded & 0.7436 / 3.2461 & 0.7182 / 1.5454 & 0.9937 / \underline{1.7427} & 0.7159 / 0.4241 & \underline{0.9822} / 1.7923 & \textbf{0.9555} / \textbf{1.1160} \\
    Quantized & 0.7266 / 3.4248 & 0.7111 / 2.0273 & 0.9929 / \textbf{1.8310} & 0.6772 / 0.4548 & \underline{0.9798} / \underline{1.8953} & \textbf{0.9400} / 1.2283 \\
    Discontinuous & 0.7430 / 6.0256 & 0.7375 / 2.3454 & 0.9939 / \underline{2.5439} & 0.7048 / 0.6711 & \textbf{0.9808} / 2.7435 & \underline{0.9842} / \textbf{2.1095} \\
    Near-zero & 0.6522 / 2.4134 & 0.6009 / 1.4672 & 0.9804 / \textbf{0.6205} & 0.7069 / 0.2473 & \underline{0.9541} / \underline{0.6456} & \textbf{0.9516} / 0.6552 \\
    High noise & 0.8822 / 13.6168 & 0.8874 / 5.8670 & 0.9991 / \underline{9.6276} & 0.7656 / 1.6825 & \underline{0.9957} / 9.7841 & \textbf{0.9709} / \textbf{4.7008} \\
    Heavy-tail noise & 0.7820 / 5.1986 & 0.7528 / 2.5593 & 0.9928 / \textbf{3.0465} & 0.7191 / 0.6445 & \underline{0.9822} / \underline{3.1927} & \textbf{0.9460} / 1.7041 \\
    Asymmetric outlier & 0.8329 / 13.5983 & 0.8085 / 5.6424 & 0.9790 / \textbf{6.8455} & 0.6712 / 1.1702 & \underline{0.9318} / 6.6426 & \textbf{0.9400} / 3.1727 \\
    Heteroscedastic & 0.7918 / 5.5041 & 0.7524 / 2.8392 & 0.9946 / \underline{3.3166} & 0.7207 / 0.7290 & \underline{0.9836} / 3.3943 & \textbf{0.9704} / \textbf{1.8532} \\
    Categorical + missing & 0.7714 / 16.2034 & 0.6868 / 2.0897 & 0.9818 / \textbf{2.4965} & 0.5716 / 0.6235 & \underline{0.9569} / \underline{2.5532} & \textbf{0.9473} / 2.0306 \\
    \bottomrule
  \end{tabular*}
\end{table}

Overall, \TabCon\ remains well calibrated across a broad range of challenging data-generating mechanisms. Its coverage is closest to the nominal level of 95\% in 10 out of the 11 DGPs, with coverage never falling below 94\%. Moreover, among the six DGPs for which \TabCon\ satisfies $\mathrm{CP}\geq0.95$, it produces the shortest interval among valid methods in five cases. Looking across individual stress-test regimes, \TabCon\ achieves nearly nominal coverage under the near-constant, bounded, and near-zero settings, suggesting that its calibration remains stable when the regression function becomes weakly varying or constrained. It also maintains good coverage under high noise and heteroscedasticity, while producing shorter intervals than PI and AMP. The most challenging settings for \TabCon\ are the quantized, heavy-tail noise, asymmetric outlier, and categorical-with-missing DGPs, where it mildly undercovers, with coverage ranging from 0.94 to 0.9473. Nevertheless, \TabCon\ remains the method closest to the nominal coverage level in each of these cases.

Similarly to the results in the Semi-Synthetic setting, AMP provides the closest comparison to \TabCon\ in terms of calibration, but tends to overcover across most stress-test regimes. It is closest to the nominal level only in the discontinuous setting and generally produces wider intervals than \TabCon. Also PI overcovers across all stress-test DGPs, with coverage ranging from approximately 0.98 to almost one, with longer intervals than those produced by \TabCon.

Finally, Linear, RFF, and Bootstrap undercover across every stress-test regime. Bootstrap produces the shortest intervals in most regimes, yet these intervals are too narrow.

\subsection{Results at 90\% nominal coverage.}\label{app:other_alpha_90} We report coverage probability and interval length at 90\% nominal level across the baseline algorithms and \TabCon\ in Table~\ref{tab:tabcon-main-results-90}. 

\begin{table}[!htbp]
  \centering
  \caption{Coverage probability (CP) and interval length (IL) at the 90\% nominal coverage level. From left to right, results are reported for the standard in-distribution (ID), standard out-of-distribution (OOD), stress-test, and CC18 settings. For CP, the value closest to the nominal level of 0.90 is shown in \textbf{bold}, and the second closest is \ul{underlined}. For IL, among methods satisfying $\mathrm{CP}\geq 0.90$, the shortest interval is shown in \textbf{bold}, and the second shortest is \ul{underlined}.}
  \label{tab:tabcon-main-results-90}

  \setlength{\tabcolsep}{5.5pt}
  \renewcommand{\arraystretch}{1.10}
  \small

  \begin{tabular}{@{}lcccccccc@{}}
    \toprule
    Method
    & \multicolumn{2}{c}{Standard (ID)}
    & \multicolumn{2}{c}{Standard (OOD)}
    & \multicolumn{2}{c}{Stress-test}
    & \multicolumn{2}{c}{CC18} \\
    \cmidrule(lr){2-3}
    \cmidrule(lr){4-5}
    \cmidrule(lr){6-7}
    \cmidrule(l){8-9}
    & CP & IL & CP & IL & CP & IL & CP & IL \\
    \midrule

    Linear
    & 0.8293 & 7.1921
    & 0.6215 & 2.1325
    & 0.7247 & 5.8131
    & 0.4926 & 65.4884 \\

    RFF
    & 0.7921 & 3.4030
    & 0.7322 & 2.1013
    & 0.6763 & 2.2159
    & 0.5720 & 0.6306 \\

    PI
    & 0.9508 & 3.2849
    & 0.9427 & 1.7829
    & 0.9747 & 2.4356
    & 0.9874 & 2.0622 \\

    Bootstrap
    & 0.5280 & 0.7770
    & 0.6240 & 0.6906
    & 0.6371 & 0.5657
    & 0.6807 & 0.3532 \\

    AMP
    & \textbf{0.9015} & \underline{2.9459}
    & \textbf{0.9013} & \underline{1.5370}
    & \underline{0.9556} & \underline{2.2527}
    & \underline{0.9685} & \underline{1.5867} \\

    TabCon
    & \underline{0.9062} & \textbf{2.0604}
    & \underline{0.9037} & \textbf{1.5017}
    & \textbf{0.9116} & \textbf{1.4154}
    & \textbf{0.9329} & \textbf{0.7704} \\

    \bottomrule
  \end{tabular}
\end{table}
Overall, the results are consistent with those observed at 95\% coverage. \TabCon\ maintains coverage close to the nominal level across all four evaluation settings while producing the shortest intervals among methods satisfying $\mathrm{CP}\geq 0.90$ in every case.

In the two standard settings, AMP achieves coverage slightly closer to the nominal level, with CP equal to 0.9015 and 0.9013 under ID and OOD evaluation, respectively, compared with 0.9062 and 0.9037 for \TabCon. In the more challenging stress-test and CC18 settings, \TabCon\ achieves coverage closer to the 90\% target than AMP, reaching 0.9116 and 0.9329, respectively, while AMP overcovers with CP equal to 0.9556 and 0.9685. This difference is accompanied by shorter intervals for \TabCon, especially in the stress-test and CC18 settings.

PI again demonstrates overcoverage across all settings, with coverage ranging from approximately 0.94 to 0.99. This conservative behavior is accompanied by wider intervals than those produced by \TabCon, with the predictive distribution capturing both estimation uncertainty and outcome variability. Linear, RFF, and Bootstrap, by contrast, undercover in every setting. Bootstrap produces the shortest intervals overall, but its coverage remains far below the nominal level. Linear and RFF similarly fail to attain the desired coverage, with the largest discrepancies occurring under OOD and CC18 evaluation.

Overall, lowering the nominal coverage level from 95\% to 90\% preserves the main empirical pattern, as \TabCon\ remains well calibrated while providing sharper intervals, especially in the more challenging stress-test and CC18 settings.

\paragraph{Semi-Synthetic CC18 Benchmark.} We report coverage probability and interval length at 90\% nominal level separately for each dataset \ul{\textit{CC18}} across the baseline algorithms and \TabCon\ in Table~\ref{tab:cc18-per-dataset-90}.

\begin{table}[!htbp]
  \centering
  \caption{Per-Dataset coverage probability (CP) and interval length (IL) on the CC18 real-covariate semi-synthetic benchmark at the 90\% nominal level. From left to right, results are reported for Linear, RFF, PI, Bootstrap, AMP and \TabCon. For CP, the value closest to the nominal level of 0.90 is shown in \textbf{bold}, and the second closest is \ul{underlined}. For IL, among methods satisfying $\mathrm{CP}\geq 0.90$, the shortest IL is shown in \textbf{bold}, and the second shortest is \ul{underlined}.}
  \label{tab:cc18-per-dataset-90}
  \setlength{\tabcolsep}{2.1pt}
  \renewcommand{\arraystretch}{1.10}
  \fontfamily{ptm}\fontsize{8pt}{9.6pt}\selectfont
  \begin{tabular*}{\textwidth}{@{\extracolsep{\fill}}lcccccc@{}}
    \toprule
    Dataset & Linear & RFF & PI & Bootstrap & AMP & \TabCon \\
    & \multicolumn{1}{c}{CP / IL} & \multicolumn{1}{c}{CP / IL} & \multicolumn{1}{c}{CP / IL} & \multicolumn{1}{c}{CP / IL} & \multicolumn{1}{c}{CP / IL} & \multicolumn{1}{c}{CP / IL} \\
    \midrule
    MiceProtein & 0.6469 / 0.8562 & 0.5719 / 0.8071 & 0.9735 / \underline{2.1530} & 0.5549 / 0.2942 & \underline{0.9580} / \textbf{1.9963} & \textbf{0.8630} / 0.7511 \\
    analcatdata-dmft & 0.7561 / 0.4431 & \underline{0.8236} / 0.9032 & 0.9987 / 1.9072 & 0.7903 / 0.5841 & 0.9802 / \underline{1.0162} & \textbf{0.9673} / \textbf{0.5092} \\
    bank-marketing & 0.4716 / 0.1399 & 0.4098 / 0.1697 & 0.9924 / 2.7990 & 0.6609 / 0.3886 & \textbf{0.9657} / \underline{2.1364} & \underline{0.9822} / \textbf{1.5906} \\
    blood-transfusion & 0.3011 / 0.1478 & 0.6152 / 0.3703 & 0.9800 / 0.7509 & 0.7648 / 0.2263 & \underline{0.9549} / \underline{0.5553} & \textbf{0.9416} / \textbf{0.1882} \\
    breast-w & 0.3659 / 0.4324 & 0.7165 / 1.2749 & 0.9871 / 2.4586 & 0.6976 / 0.5889 & \underline{0.9437} / \underline{1.7064} & \textbf{0.9433} / \textbf{1.2013} \\
    connect-4 & 0.5504 / 0.0452 & 0.3270 / 0.0474 & 0.9916 / 1.0893 & 0.7276 / 0.1334 & \underline{0.9857} / \underline{0.9911} & \textbf{0.9827} / \textbf{0.6769} \\
    credit-approval & 0.6966 / 0.5154 & 0.7449 / 0.7865 & 0.9836 / 1.7989 & 0.6556 / 0.1792 & \underline{0.9589} / \underline{1.4482} & \textbf{0.9024} / \textbf{0.4114} \\
    first-order-theorem & 0.4530 / 0.2046 & 0.4387 / 0.2013 & 0.9816 / 1.2511 & 0.6320 / 0.2722 & \underline{0.9719} / \underline{1.1667} & \textbf{0.9321} / \textbf{0.7024} \\
    jungle-chess & 0.1668 / 0.0465 & 0.3587 / 0.1165 & 0.9861 / 2.9695 & 0.7980 / 0.4106 & \underline{0.9814} / \underline{1.5380} & \textbf{0.9771} / \textbf{0.7803} \\
    kc2 & 0.6949 / 846.7635 & \underline{0.8407} / 1.6630 & 0.9898 / 1.8475 & 0.7084 / 0.3956 & 0.9815 / \underline{1.6754} & \textbf{0.9208} / \textbf{0.5889} \\
    mfeat-fourier & 0.6400 / 0.7674 & 0.5982 / 0.7642 & 0.9910 / \underline{3.3122} & 0.5755 / 0.2742 & \underline{0.9723} / \textbf{2.8306} & \textbf{0.8892} / 0.7405 \\
    pc4 & 0.4844 / 0.8880 & 0.5736 / 0.8703 & 0.9849 / \underline{2.8591} & 0.5805 / 0.5080 & \underline{0.9769} / \textbf{2.6276} & \textbf{0.8938} / 1.5460 \\
    phoneme & 0.1763 / 0.0988 & 0.4175 / 0.2229 & 0.9956 / 1.6120 & 0.7026 / 0.3361 & \underline{0.9587} / \underline{0.9394} & \textbf{0.9320} / \textbf{0.3280} \\
    \bottomrule
  \end{tabular*}
\end{table}

Comparing these results with those obtained at the 95\% nominal level reveals a largely consistent pattern across coverage levels. At 95\%, \TabCon\ achieves the coverage closest to the target in 7 out of the 13 datasets, whereas at 90\% this increases to 12 out of 13. The three datasets that remain undercovered at 90\% (i.e., \textit{MiceProtein}, \textit{mfeat-fourier}, and \textit{pc4}) are also among the four undercovered datasets at 95\%, suggesting that the remaining calibration errors are concentrated in a small and consistent subset of covariate distributions. In contrast, \textit{credit-approval} moves from mild undercoverage at 95\% to nearly exact nominal coverage at 90\%. Moreover, \TabCon\ produces the shortest interval among methods satisfying the nominal coverage requirement in most of the cases.

AMP again provides the closest comparison to \TabCon\ in terms of calibration. It satisfies the nominal coverage requirement on all 13 datasets, but overcovers, with coverage ranging from approximately 0.94 to 0.99. Moreover, whenever \TabCon\ attains the nominal coverage level, it produces shorter intervals than AMP, with large reductions on datasets such as \textit{credit-approval}, \textit{kc2}, and \textit{phoneme}.

PI shows an even more conservative behavior, producing coverage between approximately 0.97 and one across all datasets and wider intervals. In contrast, Linear, RFF, and Bootstrap undercover on every CC18 dataset. Bootstrap again produces very short intervals, but these are accompanied by coverage far below the nominal level. Linear are still characterized by an extremely large interval on \textit{kc2}.

\paragraph{Synthetic Benchmark.} We report coverage probability and interval length at 90\% nominal level separately for each stress-test DGP across the baseline algorithms and \TabCon\ in Table~\ref{tab:stress-per-regime-90}. 

\begin{table}[!htbp]
  \centering
  \caption{Per-DGP stress-test coverage probability (CP) and interval length (IL) at the 90\% nominal level. From left to right, results are reported for Linear, RFF, PI, Bootstrap, AMP and \TabCon. For CP, the value closest to the nominal level of 0.90 is shown in \textbf{bold}, and the second closest is \ul{underlined}. For IL, among methods satisfying $\mathrm{CP}\geq 0.90$, the shortest IL is shown in \textbf{bold}, and the second shortest is \ul{underlined}.}
  \label{tab:stress-per-regime-90}
  \setlength{\tabcolsep}{2.5pt}
  \renewcommand{\arraystretch}{1.12}
  \fontfamily{ptm}\fontsize{8pt}{9.6pt}\selectfont
  \begin{tabular*}{\textwidth}{@{\extracolsep{\fill}}lcccccc@{}}
    \toprule
    DGP & Linear & RFF & PI & Bootstrap & AMP & \TabCon \\
    & \multicolumn{1}{c}{CP / IL} & \multicolumn{1}{c}{CP / IL} & \multicolumn{1}{c}{CP / IL} & \multicolumn{1}{c}{CP / IL} & \multicolumn{1}{c}{CP / IL} & \multicolumn{1}{c}{CP / IL} \\
    \midrule
    Regular control & 0.7152 / 3.0867 & 0.6526 / 1.6813 & 0.9791 / 1.5268 & 0.6526 / 0.3954 & \underline{0.9643} / \underline{1.3716} & \textbf{0.9013} / \textbf{0.9402} \\
    Near-constant & 0.7043 / 0.0301 & 0.6598 / 0.0174 & 0.9842 / 0.0165 & 0.6567 / 0.0036 & \underline{0.9707} / \underline{0.0146} & \textbf{0.9131} / \textbf{0.0083} \\
    Bounded & 0.7009 / 2.7242 & 0.6534 / 1.2970 & 0.9840 / 1.4486 & 0.6561 / 0.3574 & \underline{0.9703} / \underline{1.3062} & \textbf{0.9131} / \textbf{0.8636} \\
    Quantized & 0.6851 / 2.8742 & 0.6468 / 1.7014 & 0.9839 / \underline{1.5165} & 0.6163 / 0.3861 & \underline{0.9666} / \textbf{1.3686} & \textbf{0.8886} / 0.9486 \\
    Discontinuous & 0.7023 / 5.0568 & 0.6840 / 1.9683 & 0.9879 / 2.0999 & 0.6526 / 0.5763 & \underline{0.9711} / \underline{1.9304} & \textbf{0.9679} / \textbf{1.6318} \\
    Near-zero & 0.6027 / 2.0254 & 0.5299 / 1.2313 & 0.9602 / \underline{0.4961} & 0.6504 / 0.2080 & \underline{0.9317} / \textbf{0.4375} & \textbf{0.9007} / 0.4982 \\
    High noise & \underline{0.8421} / 11.4276 & 0.8317 / 4.9238 & 0.9988 / 8.0743 & 0.7027 / 1.4357 & 0.9928 / \underline{7.2671} & \textbf{0.9338} / \textbf{3.6183} \\
    Heavy-tail noise & 0.7410 / 4.3628 & 0.6935 / 2.1478 & 0.9825 / \underline{2.4623} & 0.6530 / 0.5467 & \underline{0.9709} / \textbf{2.2473} & \textbf{0.8923} / 1.3091 \\
    Asymmetric outlier & 0.7930 / 11.4121 & 0.7504 / 4.7353 & \textbf{0.9114} / \textbf{3.4436} & 0.6061 / 0.9941 & 0.8702 / 3.5821 & \underline{0.8764} / 2.2658 \\
    Heteroscedastic & 0.7476 / 4.6192 & 0.6908 / 2.3827 & 0.9883 / 2.7476 & 0.6606 / 0.6171 & \underline{0.9748} / \underline{2.4879} & \textbf{0.9319} / \textbf{1.4162} \\
    Categorical + missing & 0.7282 / 13.5983 & 0.6226 / 1.7537 & 0.9654 / \underline{2.0503} & 0.5164 / 0.5324 & \underline{0.9364} / \textbf{1.8853} & \textbf{0.8983} / 1.5947 \\
    \bottomrule
  \end{tabular*}
\end{table}

The same challenging regimes identified at the 95\% nominal level remain visible at 90\%. Specifically, \TabCon\ mildly undercovers under quantization, heavy-tail noise, and categorical covariates with missing observations. The asymmetric-outlier setting is more challenging, with coverage decreasing to 0.8764; this is also the only DGP for which \TabCon\ is not the method closest to the nominal target. These results suggest that calibration remains more difficult under distributional irregularities that directly affect the shape or tails of the estimation error distribution. Overall, \TabCon\ remains well calibrated across a broad range of challenging data-generating mechanisms.

The behavior of the baseline methods is similar with that observed at the 95\% nominal level. AMP and PI remain conservative despite the lower coverage target, with PI overcovering across all stress-test DGPs and AMP doing so in all but the asymmetric-outlier setting. In contrast, Linear, RFF, and Bootstrap continue to undercover throughout the stress-test benchmark.

\subsection{Results at 99\% nominal coverage.}\label{app:other_alpha_99}
We report coverage probability and interval length at 99\% nominal level across the baseline algorithms and \TabCon\ in Table~\ref{tab:tabcon-main-results-99}.

\begin{table}[!htbp]
  \centering
  \caption{Coverage probability (CP) and interval length (IL) at the 99\% nominal coverage level. From left to right, 
    results are reported for the standard in-distribution (ID), standard out-of-distribution (OOD), stress-test, and CC18 settings. For CP, the value closest to the nominal level of 0.99 is shown in \textbf{bold}, and the second closest is \ul{underlined}. For IL, among methods satisfying $\mathrm{CP}\geq 0.99$, the shortest interval is shown in \textbf{bold}, and the second shortest is \ul{underlined}.}
  \label{tab:tabcon-main-results-99}

  \setlength{\tabcolsep}{5.5pt}
  \renewcommand{\arraystretch}{1.10}
  \small

  \begin{tabular}{@{}lcccccccc@{}}
    \toprule
    Method
    & \multicolumn{2}{c}{Standard (ID)}
    & \multicolumn{2}{c}{Standard (OOD)}
    & \multicolumn{2}{c}{Stress-test}
    & \multicolumn{2}{c}{CC18} \\
    \cmidrule(lr){2-3}
    \cmidrule(lr){4-5}
    \cmidrule(lr){6-7}
    \cmidrule(l){8-9}
    & CP & IL & CP & IL & CP & IL & CP & IL \\
    \midrule

    Linear
    & 0.9086 & 11.2627
    & 0.7681 & 3.3396
    & 0.8189 & 9.1032
    & 0.5977 & 102.5543 \\

    RFF
    & 0.9228 & 5.3290
    & 0.8965 & 3.2906
    & 0.8219 & 3.4700
    & 0.6911 & 0.9874 \\

    PI
    & \underline{0.9966} & \underline{5.6779}
    & \underline{0.9925} & \underline{3.1370}
    & 0.9977 & \underline{4.9424}
    & 0.9993 & \underline{3.8716} \\

    Bootstrap
    & 0.6705 & 1.1378
    & 0.7499 & 0.9902
    & 0.7598 & 0.8210
    & 0.7890 & 0.4897 \\

    AMP
    & 0.9706 & 13.1192
    & 0.9645 & 6.5085
    & \underline{0.9866} & 10.3262
    & \textbf{0.9918} & 7.5406 \\

    TabCon
    & \textbf{0.9923} & \textbf{4.3173}
    & \textbf{0.9914} & \textbf{3.0418}
    & \textbf{0.9918} & \textbf{3.0789}
    & \underline{0.9919} & \textbf{1.8936} \\

    \bottomrule
  \end{tabular}
\end{table}

At 99\% nominal coverage level, \TabCon\ continues to closely track the prescribed target across all evaluation settings. It achieves coverage closest to 0.99 in the standard ID, standard OOD, and stress-test settings, while ranking second on CC18 by a negligible margin. Importantly, \TabCon\ satisfies $\mathrm{CP}\geq0.99$ in all four settings and produces the shortest interval among methods meeting the nominal coverage requirement.

The behavior of AMP changes noticeably at this higher nominal level. While it remains well calibrated on CC18, it undercovers in the ID, OOD, and stress-test settings, in contrast with the more conservative behavior observed at the 90\% and 95\% levels. PI, instead, continues to overcover across all four settings, preserving the same conservative pattern observed at lower nominal coverage levels.

Linear, RFF, and Bootstrap again fail to attain the desired coverage in every setting. Their undercoverage becomes even more apparent at the 99\% target, while Bootstrap continues to achieve short intervals at the cost of substantial miscalibration.

\paragraph{Semi-Synthetic CC18 Benchmark.} We report coverage probability and interval length at 99\% nominal level separately for each dataset \ul{\textit{CC18}} across the baseline algorithms and \TabCon\ in Table~\ref{tab:cc18-per-dataset-99}.

\begin{table}[!htbp]
  \centering
  \caption{Per-Dataset coverage probability (CP) and interval length (IL) on the CC18 real-covariate semi-synthetic benchmark at the 99\% nominal level. From left to right, results are reported for Linear, RFF, PI, Bootstrap, AMP and \TabCon. For CP, the value closest to the nominal level of 0.99 is shown in \textbf{bold}, and the second closest is \ul{underlined}. For IL, among methods satisfying $\mathrm{CP}\geq 0.99$, the shortest IL is shown in \textbf{bold}, and the second shortest is \ul{underlined}.}
  \label{tab:cc18-per-dataset-99}
  \setlength{\tabcolsep}{2.1pt}
  \renewcommand{\arraystretch}{1.10}
  \fontfamily{ptm}\fontsize{8pt}{9.6pt}\selectfont
  \begin{tabular*}{\textwidth}{@{\extracolsep{\fill}}lcccccc@{}}
    \toprule
    Dataset & Linear & RFF & PI & Bootstrap & AMP & \TabCon \\
    & \multicolumn{1}{c}{CP / IL} & \multicolumn{1}{c}{CP / IL} & \multicolumn{1}{c}{CP / IL} & \multicolumn{1}{c}{CP / IL} & \multicolumn{1}{c}{CP / IL} & \multicolumn{1}{c}{CP / IL} \\
    \midrule
    MiceProtein & 0.7840 / 1.3407 & 0.7432 / 1.2639 & \underline{0.9978} / \textbf{3.4587} & 0.6642 / 0.4275 & \textbf{0.9883} / 8.4613 & 0.9818 / 1.9110 \\
    analcatdata-dmft & 0.8644 / 0.6939 & 0.9396 / 1.4144 & 1.0000 / \underline{3.6858} & 0.9003 / 0.7824 & \textbf{0.9937} / 5.7086 & \underline{0.9975} / \textbf{1.8055} \\
    bank-marketing & 0.5471 / 0.2191 & 0.4903 / 0.2657 & 0.9998 / \underline{5.6640} & 0.7550 / 0.5360 & \textbf{0.9913} / 10.7802 & \underline{0.9995} / \textbf{3.6229} \\
    blood-transfusion & 0.4075 / 0.2314 & 0.7808 / 0.5799 & 0.9987 / \underline{2.5471} & 0.8650 / 0.3091 & \textbf{0.9888} / 2.9099 & \underline{0.9955} / \textbf{0.4864} \\
    breast-w & 0.4976 / 0.6772 & 0.8380 / 1.9965 & 0.9995 / \underline{4.2610} & 0.8097 / 0.8292 & \underline{0.9814} / 8.6391 & \textbf{0.9957} / \textbf{2.7535} \\
    connect-4 & 0.6460 / 0.0707 & 0.4072 / 0.0742 & \underline{0.9992} / \underline{1.7785} & 0.8260 / 0.1883 & \textbf{0.9953} / 4.5193 & 0.9994 / \textbf{1.4232} \\
    credit-approval & 0.8242 / 0.8071 & 0.8845 / 1.2317 & \underline{0.9986} / \textbf{2.9537} & 0.7739 / 0.2627 & \textbf{0.9899} / 6.6121 & 0.9729 / 1.1641 \\
    first-order-theorem & 0.5629 / 0.3205 & 0.5499 / 0.3152 & 0.9989 / \underline{2.1894} & 0.7433 / 0.3598 & \textbf{0.9934} / 5.7711 & \underline{0.9971} / \textbf{1.6633} \\
    jungle-chess & 0.2189 / 0.0728 & 0.4610 / 0.1824 & 1.0000 / \underline{5.1122} & 0.8923 / 0.6057 & \textbf{0.9966} / 7.0951 & \underline{0.9998} / \textbf{2.3391} \\
    kc2 & 0.8205 / 1326.0257 & 0.9280 / 2.6043 & 0.9994 / \underline{3.8094} & 0.8284 / 0.5536 & \textbf{0.9949} / 8.2706 & \underline{0.9974} / \textbf{1.4891} \\
    mfeat-fourier & 0.7685 / 1.2017 & 0.7158 / 1.1968 & \underline{0.9993} / \textbf{6.6904} & 0.6778 / 0.3866 & \textbf{0.9957} / \underline{12.4504} & 0.9727 / 1.8861 \\
    pc4 & 0.5758 / 1.3907 & 0.7034 / 1.3629 & 1.0000 / \textbf{4.8660} & 0.6890 / 0.6901 & \underline{0.9954} / \underline{11.9363} & \textbf{0.9874} / 3.2021 \\
    phoneme & 0.2525 / 0.1547 & 0.5429 / 0.3490 & 0.9999 / \underline{3.3151} & 0.8321 / 0.4354 & \textbf{0.9890} / 4.8743 & \underline{0.9981} / \textbf{0.8710} \\
    \bottomrule
  \end{tabular*}
\end{table}

\TabCon\ satisfies the nominal coverage requirement in 9 out of the 13 CC18 datasets and, in every one of these cases, produces the shortest interval. It still shows undercoverage for \textit{MiceProtein}, \textit{credit-approval}, \textit{mfeat-fourier}, and \textit{pc4} datasets, as similarly highlighted in both 90\% and 95\% nominal levels.

At this higher confidence level, however, the per-dataset calibration pattern differs from that observed at 90\% and 95\%. In most of the cases, AMP is closest to the 99\% target.

PI continues to express the most conservative behavior, with coverage close to one across all datasets and longer intervals. AMP shows considerably better calibration than PI at the 99\% level, but its intervals are often much wider than those of \TabCon. Linear, RFF, and Bootstrap again undercover throughout the benchmark, with the gap from the nominal target becoming pronounced under this higher nominal coverage requirement.

\paragraph{Synthetic Benchmark.} We report coverage probability and interval length at 99\% nominal level separately for each stress-test DGP across the baseline algorithms and \TabCon\ in Table~\ref{tab:stress-per-regime-99}.

\begin{table}[!htbp]
  \centering
  \caption{Per-DGP stress-test coverage probability (CP) and interval length (IL) at the 99\% nominal level. From left to right, results are reported for Linear, RFF, PI, Bootstrap, AMP and \TabCon. For CP, the value closest to the nominal level of 0.99 is shown in \textbf{bold}, and the second closest is \ul{underlined}. For IL, among methods satisfying $\mathrm{CP}\geq 0.99$, the shortest IL is shown in \textbf{bold}, and the second shortest is \ul{underlined}.}
  \label{tab:stress-per-regime-99}
  \setlength{\tabcolsep}{2.5pt}
  \renewcommand{\arraystretch}{1.12}
  \fontfamily{ptm}\fontsize{8pt}{9.6pt}\selectfont
  \begin{tabular*}{\textwidth}{@{\extracolsep{\fill}}lcccccc@{}}
    \toprule
    DGP & Linear & RFF & PI & Bootstrap & AMP & \TabCon \\
    & \multicolumn{1}{c}{CP / IL} & \multicolumn{1}{c}{CP / IL} & \multicolumn{1}{c}{CP / IL} & \multicolumn{1}{c}{CP / IL} & \multicolumn{1}{c}{CP / IL} & \multicolumn{1}{c}{CP / IL} \\
    \midrule
    Regular control & 0.8099 / 4.8338 & 0.8032 / 2.6329 & 0.9973 / \textbf{2.5927} & 0.7749 / 0.6000 & \underline{0.9875} / 6.1256 & \textbf{0.9882} / 1.9779 \\
    Near-constant & 0.8005 / 0.0471 & 0.8102 / 0.0273 & 0.9985 / \textbf{0.0272} & 0.7774 / 0.0053 & \textbf{0.9905} / \underline{0.0639} & \underline{0.9894} / 0.0185 \\
    Bounded & 0.7970 / 4.2661 & 0.8030 / 2.0310 & 0.9985 / \textbf{2.4072} & 0.7769 / 0.5372 & \textbf{0.9905} / \underline{5.5994} & \underline{0.9894} / 1.7959 \\
    Quantized & 0.7835 / 4.5009 & 0.8012 / 2.6643 & 0.9984 / \textbf{2.5762} & 0.7403 / 0.5771 & \textbf{0.9898} / 6.0287 & \underline{0.9894} / 1.9803 \\
    Discontinuous & 0.7894 / 7.9189 & 0.8134 / 3.0824 & 0.9984 / \underline{3.6697} & 0.7613 / 0.8062 & \textbf{0.9902} / 9.2598 & \underline{0.9952} / \textbf{3.4469} \\
    Near-zero & 0.7146 / 3.1717 & 0.7096 / 1.9282 & \underline{0.9959} / \textbf{0.9192} & 0.7619 / 0.3079 & 0.9737 / 2.5396 & \textbf{0.9922} / \underline{1.0857} \\
    High noise & 0.9221 / 17.8955 & 0.9453 / 7.7106 & 0.9999 / \underline{12.8052} & 0.8288 / 2.0432 & \underline{0.9986} / 29.1379 & \textbf{0.9971} / \textbf{7.6297} \\
    Heavy-tail noise & 0.8304 / 6.8321 & 0.8374 / 3.3635 & 0.9987 / \textbf{4.4642} & 0.7857 / 0.7842 & \underline{0.9906} / \underline{10.4412} & \textbf{0.9899} / 2.7916 \\
    Asymmetric outlier & 0.8846 / 17.8712 & 0.8824 / 7.4154 & \underline{0.9934} / \textbf{14.0531} & 0.7432 / 1.4746 & 0.9717 / 21.9920 & \textbf{0.9894} / 5.8064 \\
    Heteroscedastic & 0.8413 / 7.2336 & 0.8335 / 3.7314 & 0.9990 / \underline{4.6068} & 0.7849 / 0.9011 & \textbf{0.9929} / 10.2364 & \underline{0.9948} / \textbf{3.0573} \\
    Categorical + missing & 0.8258 / 21.2948 & 0.7826 / 2.7464 & \underline{0.9965} / \underline{3.8956} & 0.6372 / 0.7730 & 0.9774 / 7.9628 & \textbf{0.9917} / \textbf{3.1763} \\
    \bottomrule
  \end{tabular*}
\end{table}
Compared to the previous cases (i.e., 90\% and 95\%), calibration at the 99\% is stable across the individual stress-test DGPs. The coverage of \TabCon\ remains tightly concentrated around the nominal target. Instead, AMP exhibits more pronounced undercoverage in settings such as near-zero signal, asymmetric outliers, and categorical covariates with missing values. PI, in contrast, preserves the conservative behavior observed at lower nominal levels and remains above 0.99 throughout the benchmark.

Linear, RFF, and Bootstrap continue to undercover across all stress-test regimes. 

\subsection{Runtime Efficiency}
\label{app:runtime-decomposition}

\makeatletter
\setlength{\@fptop}{0pt}
\makeatother

\begin{table}[!htbp]
  \centering
  \caption{Add-on and Core median runtime in seconds, over 30 repetitions. From left to right, results are reported for PI, Global, and \TabCon. Shared executor denotes the median runtime of the \TabPFN\ backbone when evaluated using the common execution pipeline. $C$ and $Q$ denote the numbers of context and query points.}
  \label{tab:tabcon-shared-executor-runtime}

  \setlength{\tabcolsep}{3.2pt}
  \renewcommand{\arraystretch}{1.05}
  \scriptsize

  \begin{tabular}{@{}rrrrrrrrr@{}}
    \toprule
    $C$ & $Q$ & Shared executor
    & \multicolumn{2}{c}{PI}
    & \multicolumn{2}{c}{Global}
    & \multicolumn{2}{c}{\TabCon} \\
    \cmidrule(lr){4-5}
    \cmidrule(lr){6-7}
    \cmidrule(l){8-9}
    & & & Add-on & Core & Add-on & Core & Add-on & Core \\
    \midrule

    64 & 64
    & 0.1673 & 0.0024 & 0.1697
    & 0.0033 & 0.1705 & 0.0060 & 0.1733 \\
    & 256
    & 0.1686 & 0.0024 & 0.1711
    & 0.0038 & 0.1725 & 0.0074 & 0.1760 \\
    & 1{,}024
    & 0.1727 & 0.0084 & 0.1811
    & 0.0128 & 0.1855 & 0.0257 & 0.1984 \\

    \midrule
    128 & 64
    & 0.1680 & 0.0024 & 0.1705
    & 0.0033 & 0.1713 & 0.0060 & 0.1741 \\
    & 256
    & 0.1691 & 0.0025 & 0.1716
    & 0.0038 & 0.1730 & 0.0074 & 0.1766 \\
    & 1{,}024
    & 0.1758 & 0.0084 & 0.1842
    & 0.0130 & 0.1888 & 0.0260 & 0.2019 \\

    \midrule
    256 & 64
    & 0.1704 & 0.0025 & 0.1729
    & 0.0033 & 0.1737 & 0.0061 & 0.1765 \\
    & 256
    & 0.1701 & 0.0025 & 0.1726
    & 0.0039 & 0.1740 & 0.0075 & 0.1777 \\
    & 1{,}024
    & 0.1760 & 0.0084 & 0.1843
    & 0.0131 & 0.1891 & 0.0262 & 0.2022 \\

    \midrule
    512 & 64
    & 0.1729 & 0.0024 & 0.1754
    & 0.0033 & 0.1762 & 0.0060 & 0.1790 \\
    & 256
    & 0.1740 & 0.0025 & 0.1765
    & 0.0039 & 0.1779 & 0.0075 & 0.1815 \\
    & 1{,}024
    & 0.1791 & 0.0085 & 0.1874
    & 0.0132 & 0.1923 & 0.0265 & 0.2055 \\

    \midrule
    1{,}024 & 64
    & 0.1762 & 0.0024 & 0.1786
    & 0.0033 & 0.1795 & 0.0060 & 0.1822 \\
    & 256
    & 0.1793 & 0.0025 & 0.1818
    & 0.0039 & 0.1832 & 0.0075 & 0.1868 \\
    & 1{,}024
    & 0.1816 & 0.0085 & 0.1900
    & 0.0133 & 0.1948 & 0.0267 & 0.2082 \\

    \bottomrule
  \end{tabular}
  
\end{table}

Given the runtime results reported in Section~\ref{sec:experiment}, we further investigate the computational overhead introduced by \TabCon. While the previous results in Table \ref{tab:runtime-grid-new-task} reports end-to-end latency under the native implementation of each method, here we consider another view of the runtime performances in which PI, Global, and \TabCon\ share the same low-level \TabPFN\ backbone computation.

The \emph{Shared executor} column reports the common backbone latency, while \emph{Add-on} measures the additional computation performed after the shared outputs are available. Specifically, this includes distribution decoding for PI, one residual head for Global, and routing together with two residual heads for \TabCon. The \emph{Core} runtime is given by the sum of  \emph{Shared executor} and \emph{Add-on}.

\TabCon\ introduces a larger add-on cost than PI and Global, reflecting the additional routing and residual-head computations, with the difference becoming more visible as the number of query points increases. Nevertheless, the shared \TabPFN\ execution remains by far the dominant component of the total runtime. The add-on measurements further show that the additional computation introduced by \TabCon\ is small in absolute terms. Compared with Global, adding the router and evaluating a second residual head increases the add-on latency by only a few milliseconds.

\end{document}